\documentclass[12pt]{article}
\newlength{\defbaselineskip}
\usepackage[a4paper , total={8.5in, 11in}, margin=1in]{geometry}

\usepackage{hyperref}
\hypersetup{
     colorlinks   = true,
     linkcolor    = blue,
     citecolor    = green
}

\usepackage{graphicx}
\usepackage{subfigure}
\usepackage{booktabs} 
\usepackage{epstopdf}
\usepackage[all]{xy}
\usepackage{mathptmx}      
\usepackage{amssymb,mathrsfs,amsmath,amscd,amsfonts,amssymb,bm,color}
\usepackage[mathscr]{euscript}
\usepackage{stackrel}
\usepackage{fancyhdr}
\usepackage{makeidx}
\usepackage{afterpage}

\DeclareMathOperator*{\minimize}{minimize}

\DeclareMathOperator*{\argmin}{arg\,min}

\newcommand{\nonl}{\renewcommand{\nl}{\let\nl\oldnl}}
\newcommand{\e}[1]{\mathbb{E}\left[#1\right]}
\newcommand{\F}{\mathcal{F}}
\def\1{\bm{1}}
\newcommand{\degree}{{\mathrm{deg}}}
\newcommand{\norm}[1]{\left\lVert#1\right\rVert}

\usepackage{amsthm,bm}
\newtheorem{thm}{Theorem}

\newtheorem{defn}[thm]{Definition}
\newtheorem{cor}[thm]{Corollary}

\newtheorem{rem}[thm]{Remark}
\newtheorem{lem}[thm]{Lemma}

\newtheorem{assum}[thm]{Assumption}

\usepackage[linesnumbered,ruled,vlined,commentsnumbered]{algorithm2e}

\usepackage{caption}
\usepackage[hang,flushmargin]{footmisc}
\usepackage{lipsum}
\makeatletter
\newcommand{\algorithmfootnote}[2][\footnotesize]{%
  \let\old@algocf@finish\@algocf@finish
  \def\@algocf@finish{\old@algocf@finish
    \leavevmode\rlap{\begin{minipage}{\linewidth}
    #1#2
    \end{minipage}}%
  }%
}

\begin{document}

\title{DADAM: A Consensus-based Distributed \\Adaptive Gradient Method for Online Optimization}
	\vspace{0.8cm}
\author{{Parvin Nazari}\thanks{Department of Mathematics \& Computer Science, Amirkabir University of Technology, Email: \texttt{p$\_$nazari@aut.ac.ir}}\and{Davoud Ataee Tarzanagh }\thanks{ Department of Mathematics \& UF Informatics Institute, University of Florida, Email: \texttt{tarzanagh@ufl.edu}
}
\and{George Michailidis}\thanks{ Department of Statistics \& UF Informatics Institute, University of Florida, Email: \texttt{gmichail@ufl.edu}
}
}
\date{}
	\maketitle
	
\title{}
	\maketitle
Adaptive gradient-based optimization methods such as \textsc{Adagrad}, \textsc{Rmsprop}, and \textsc{Adam} are widely used in solving large-scale machine learning problems including deep learning. A number of schemes have been proposed in the literature aiming at parallelizing them, based on communications of peripheral nodes with a central node, but incur high communications cost. To address this issue, we develop a novel consensus-based distributed adaptive moment estimation method (\textsc{Dadam}) for online optimization over a decentralized network that enables data parallelization, as well as decentralized computation. The method is particularly useful, since it can accommodate settings where access to local data is allowed. Further, as established theoretically in this work, it can outperform centralized adaptive algorithms, for certain classes of loss functions used in applications. We analyze the convergence properties of the proposed algorithm and provide a dynamic regret bound on the convergence rate of adaptive moment estimation methods in both stochastic and deterministic settings. Empirical results demonstrate that \textsc{Dadam} works also well in practice and compares favorably to competing online optimization methods.
\\
\\
Adaptive gradient method. Online learning. Distributed optimization. Regret minimization.
\section{Introduction}

Online optimization is a fundamental procedure for solving a wide range of machine learning problems \cite{shalev2012online,hazan2016introduction}. It can be formulated as a repeated game between a learner (algorithm) and an adversary. The learner receives a streaming data sequence, sequentially selects actions, and the adversary reveals the convex or nonconvex losses to the learner. A standard performance metric for an online algorithm is \textit{regret}, which measures the performance of the algorithm versus a static benchmark \cite{zinkevich2003online,hazan2016introduction}. For example, the benchmark could be an optimal point of the online average of the loss (local cost) function, had the learner known all the losses in advance. In a broad sense, if the benchmark is a fixed sequence, the regret is called \textit{static}. Recent work on online optimization has investigated the notion of \textit{dynamic} regret \cite{zinkevich2003online, hall2015online, besbes2015non}. Dynamic regret can take the form of the cumulative difference between the instantaneous loss and the minimum loss. For convex functions, previous studies have shown that the dynamic regret of online gradient-based methods can be upper bounded by $O(\sqrt{T D_{T}})$, where $D_{T}$  is a measure of regularity of the comparator sequence or the function sequence \cite{zinkevich2003online,hall2015online,shahrampour2018distributed}. This bound can be improved to $O(D_{T})$ \cite{mokhtari2016online,zhang2017improved}, when the cost function is strongly convex and smooth.

Decentralized nonlinear programming has received a lot of
interest in diverse scientific and engineering fields \cite{tsitsiklis1986distributed,li2002detection,rabbat2004distributed, lesser2012distributed}. The key problem involves optimizing a cost function $f(x)=\frac{1}{n}\sum_{i=1}^n f_i(x)$, where $x\in{\mathbb{R}}^p$ and each $f_i$ is only known to the individual agent $i$ in a connected network of $n$ agents. The agents collaborate by successively sharing information with other agents located in their neighborhood with the goal of jointly converging to the network-wide optimal argument \cite{nedic2009distributed}. Compared to optimization procedures involving a fusion center that collects data and performs the computation, decentralized nonlinear programming enjoys the advantage of scalability to the size of the network used, robustness to the network topology, and privacy preservation in data-sensitive applications.

A popular algorithm in decentralized optimization is gradient descent which has been studied in \cite{nedic2009distributed,duchi2012dual}. Convergence results for convex problems with bounded gradients are given in \cite{yuan2016convergence}, while analogous convergence results even for nonconvex problems are given in \cite{zeng2018nonconvex}. Convergence can be accelerated by using corrected update rules and momentum techniques \cite{shi2015extra,nedic2017achieving,tang2018d,jiang2017collaborative}. The primal-dual \cite{chang2014distributed,lan2017communication}, \textsc{Admm} \cite{wei2012distributed,shi2014linear} and zero-order \cite{hajinezhad2017zeroth} approaches are related to the dual decentralized gradient method \cite{duchi2012dual}. We also point out recent work on a very efficient consensus-based decentralized stochastic gradient (\textsc{Dsgd}) method for deep learning over fixed topology networks \cite{jiang2017collaborative} and earlier work on decentralized gradient methods for nonconvex deep learning problems \cite{lian2017can}.  Further, under some mild assumptions, \cite{lian2017can} shows that decentralized algorithms can be faster than their centralized counterparts for certain stochastic nonconvex loss functions.

Appropriately choosing the learning rate that scales coordinates
of the gradient and the way of updating them are critical issues that impact the performance of first \cite{duchi2011adaptive,kingma2014ADAM,tielemandivide} and second order optimization procedures \cite{zhang2002adaptive,ataee2014new,tarzanagh2015new}. Indeed, an adaptive learning rate is advantageous, which led to the development of a family of widely-used methods including \textsc{Adagrad} \cite{duchi2011adaptive}, \textsc{Adadelta}~\cite{zeiler2012adadelta}, \textsc{Rmsprop} \cite{tielemandivide}, \textsc{Adam} \cite{kingma2014ADAM} and \textsc{Amsgrad} \cite{reddi2018convergence}.  Numerical results show that \textsc{Adam} can achieve significantly better performance compared to \textsc{Adagrad}, \textsc{Adadelta}, and \textsc{Rmsprop} for minimizing non-stationary objectives and problems with very noisy and/or sparse gradients. However, it has been recently demonstrated that \textsc{Adam} can fail to converge even in simple convex settings \cite{wilson2017marginal,reddi2018convergence}. To tackle this issue, some sufficient conditions such as decreasing the learning rate \cite{reddi2018convergence,zhou2018convergence,chen2018convergence,ward2018adagrad,de2018convergence} or adopting a big batch size \cite{basu2018convergence,zaheer2018adaptive} have been proposed to provide convergence guarantees for \textsc{Adam} and its variants.

In this paper, we develop and analyze a new consensus-based distributed adaptive moment estimation (\textsc{Dadam}) method that incorporates decentralized optimization and leverages a variant of adaptive moment estimation methods \cite{duchi2011adaptive,kingma2014ADAM,mcmahan2010adaptive}.
Existing distributed stochastic and adaptive gradient methods for deep learning are mostly designed for a central network topology \cite{dean2012large,li2014scaling}. The main bottleneck of such a topology lies on the communication overload on the central node, since all nodes need to concurrently communicate with it.
Hence, performance can be significantly degraded when network bandwidth is limited. These considerations motivate us to study an adaptive algorithm for network topologies, where all nodes can only communicate with their neighbors and none of the nodes is designated as ``central". Therefore, the proposed method is suitable for large scale machine learning problems, since it enables both data parallelization and decentralized computation.

Next, we briefly summarize the main technical contributions of the work.
\begin{itemize}
\item[-] The first main result (Theorem~\ref{regret}) provides guarantees of \textsc{Dadam} for constrained convex minimization problems defined over a closed convex set $\mathcal{X}$. We provide the convergence bound in terms of \textit{dynamic} regret and show that when the data features are sparse and have bounded gradients, our algorithm's regret bound can be considerably better than the ones provided by standard mirror descent and gradient descent methods \cite{nedic2009distributed,hall2015online, besbes2015non}. It is worth mentioning that the regret bounds provided for adaptive gradient methods \cite{duchi2011adaptive} are static and our results generalize them to dynamic settings.
\item[-] Theorem~\ref{C} provides a novel \textit{local regret} analysis for distributed online gradient-based algorithms for constrained nonconvex minimization problems computed over a network of agents. Specifically, we prove that under certain regularity conditions, \textsc{Dadam} can achieve a local regret bound of order $\tilde{O}(\frac{1}{T})$ for nonconvex distributed optimization. To the best of our knowledge, rigorous extensions of existing adaptive gradient methods to the distributed nonconvex setting considered in this work do not seem to be available.
 \item [-] In this paper, we also present regret analysis for distributed  optimization problems computed over a network of agents. Theorems \ref{theorem2} and \ref{theorem3} provide regret bounds of \textsc{Dadam} for problem \eqref{125} with stochastic gradients and indicate that the result of Theorems \ref{regret} and \ref{C} hold true in expectation. Further, in Corollary~\ref{Corollary10} we show that \textsc{Dadam} can achieve a local regret bound of order $O(\frac{\xi^2}{\sqrt{nT}} +\frac{1}{T})$ for nonconvex stochastic optimization where $\xi$ is an upper bound on the variance of the stochastic gradients. Hence, \textsc{Dadam} outperforms centralized adaptive algorithms such as \textsc{Adam} for certain realistic classes of loss functions when $T$ is sufficiently large.
\end{itemize}

Note that the technical results established exhibit differences from those in \cite{nedic2009distributed,duchi2012dual,yuan2016convergence,zeng2018nonconvex} with the notion of adaptive constrained optimization in online and dynamic settings.

The remainder of the paper is organized as follows. Section~\ref{sec1} gives a detailed description of \textsc{Dadam}, while Section~\ref{sec:conv}  establishes its theoretical results. Section \ref{sec:ex} explains
a network correction technique for our proposed algorithm. Section~\ref{sec4} illustrates the proposed framework on a number of synthetic and real data sets. Finally, Section~\ref{sec5} concludes the paper.

The detailed proofs of the main results established are delegated to the Supplementary Material.

\subsection{Mathematical Preliminaries and Notations.} Throughout the paper, $\mathbb{R}^p$ denotes the $p$-dimensional real space. For any pair of vectors $x,y\in \mathbb{R}^p,\,\, \langle x , y \rangle$ indicates the standard Euclidean inner product. We denote the $\ell_1$ norm by $\|X\|_1 = \sum_{ij}|x_{ij}|$, the infinity norm by $\|X\|_{\infty} = \max_{ij}|x_{ij}|$, and the Euclidean norm by $\|X\|= \sqrt{\sum_{ij}|x_{ij}|^2}$. The above norms reduce to the vector norms if $X$ is a vector. The diameter of the set $\mathcal{X}$ is given by
\begin{equation} \label{eq:diam}
\gamma_{\infty}= \sup_{x,y\in \mathcal{X}}\|x-y\|_{\infty} .
\end{equation}
Let $\mathcal{S}_+^p$ be the set of all positive definite $p\times p$ matrices. $\Pi_{\mathcal{X},A}~[x]$ denotes the Euclidean projection of a vector $x$ onto $\mathcal{X}$ for $A\in \mathcal{S}_+^p$:
\begin{equation*}\label{eq:proj}
\Pi_{\mathcal{X},A}~\big[x\big]=\argmin_{y\in \mathcal{X}} \|A^{\frac{1}{2}}(x-y)\|.
\end{equation*}
The subscript $t$ is often used to denote the time step while $y_{i,t,d}$ stands for the $d$-th element of $y_{i,t}$. Further, $y_{i,1:t,d}\in \mathbb{R}^t$ is given by $$y_{i,1:t,d}=[y_{i,1,d},y_{i,2,d},\ldots, y_{i,t,d}]^\top.$$
We let $g_{i,t}$ denote the gradient of $f$ at $x_{i,t}$. The $i$-th  largest singular value of matrix $X$ is denoted by $\sigma_i(X)$. We denote the element in the $i$-th row and $j$-th column of matrix $X$ by $[X]_{ij}$. In several theorems, we consider a connected undirected graph $\mathcal{G} = (\mathcal{V}, \mathcal{E})$ with nodes $\mathcal{V} = \{1,2,\ldots,n\}$  and edges $\mathcal{E}$. The matrix $W \in \mathbb{R}^{n \times n}$ is often used to denote the weighted adjacency matrix of graph $\mathcal{G}$. The Hadamard (entrywise) and Kronecker product are denoted by $\odot$ and $\otimes$, respectively. Finally, the expectation operator is denoted by $\mathbb{E}$.

\section{Problem Formulation and Algorithm}\label{sec1}

We develop a new online adaptive optimization method (\textsc{Dadam}) that employs data parallelization and decentralized computation over a network of agents. Given a connected undirected graph $\mathcal{G} = (\mathcal{V}, \mathcal{E})$, we let each node $i \in  \mathcal{V} $ at time $t \in [T]	\equiv \{1,\ldots,T\}$ holds its own measurement and training data $b_i$, and set $f_{i,t}(x)=\frac{1}{b_i}\sum_{j=1}^{b_i}  f_{i,t}^j (x)$. We also let each agent $i$ holds a local copy of the global variable $x$ at time $t \in [T]$, which is denoted by $x_{i,t} \in \mathbb{R}^p$. With this setup, we present a distributed adaptive gradient method for solving the minimization problem
\begin{equation}\label{125}
\minimize_{x \in \mathcal{X}} \quad F(x) =\frac{1}{n}\sum_{t=1}^T\sum_{i=1}^n f_{i,t}(x),
\end{equation}
where $f_{i,t}:\mathcal{X}\rightarrow \mathbb{R}$ is a continuously differentiable mapping on the convex set $\mathcal{X}$.

\textsc{Dadam} uses a new distributed adaptive gradient method in which a group of $n$ agents aim to solve a sequential version of problem~\eqref{125}. Here, we assume that each component function $f_{i,t}:\mathcal{X}\rightarrow \mathbb{R}$ becomes {\em only} available to agent $i\in \mathcal{V}$, after having made its decision at time $t \in [T]$. In the $t$-th step, the $i$-th agent chooses a point $x_{i,t}$ corresponding to what it considers as a good selection for the network as a whole. After committing to this choice, the agent has access to a cost function $f_{i,t}: \mathcal{X}\rightarrow \mathbb{R}$ and the network cost is then given by $f_{t}(x)= \frac{1}{n}\sum_{i=1}^nf_{i,t}(x)$. Note that this function is not known to any of the agents and is not available at any single location.

The procedure of our proposed method is
outlined in Algorithm~\ref{alg}.

\IncMargin{1.5em}
\begin{algorithm}[t]
\caption{A new distributed adaptive moment estimation method (\textsc{Dadam}).}\label{alg}
  \algorithmfootnote{\begin{itemize}
  \item [-] Good default settings are  $\beta_1 =\beta_3 = 0.9$, $\beta_2=0.999$, and $\alpha_t = \sqrt{\frac{1-\sigma_2(W)}{t}}$ where $1-\sigma_2(W)$ is the spectral gap of a doubly stochastic matrix $W$ (see, Theorem~\ref{regret}).
  \item[-] A mini-batch of stochastic gradients can be used in Line~4 for stochastic problems.
\end{itemize}}
\SetKwInOut{Input}{input}
\SetKwInOut{Output}{output}
\SetKwFor{PerPar}{parfor}{do}{fine per}
\Input{~ $x_{1}\in \mathcal{X}$, step-sizes $\{\alpha_t\}_{t=1}^{T}$, decay parameters $ \beta_{1}, \beta_{2}, \beta_{3} \in [0,1)$ and a mixing matrix $W$ satisfying \eqref{W}\;}
\BlankLine
for all $i \in \mathcal{V}$, initialize moment vectors $m_{i,0}=  \upsilon_{i,0}= \widehat{\upsilon}_{i,0} = 0$ and $x_{i,1} = x_1$\;
\For{$t\leftarrow 1$ \KwTo $T$}{\label{forins}
\For{$i \in \mathcal{V}$}{
$g_{i,t} = \nabla f_{i,t}(x_{i,t})$\;
$m_{i,t}= \beta_{1} m_{i,t-1}+(1-\beta_{1})g_{i,t}$\; 
$\upsilon_{i,t}= \beta_{2}\upsilon_{i,t-1}+(1-\beta_{2}) g_{i,t} \odot g_{i,t}$\;
  $\widehat{\upsilon}_{i,t}=\beta_{3}  \widehat{\upsilon}_{i,t-1}+(1-\beta_{3})\max(\widehat{\upsilon}_{i,t-1},\upsilon_{i,t})$\;
 $x_{i,t+\frac{1}{2}}=\sum_{j=1}^n[W]_{ij}x_{j,t}$\;
 $x_{i,t+1} = \Pi_{\mathcal{X},\sqrt{{\text{diag}(\widehat{\upsilon}_{i,t})}}}~ \big[x_{i,t+\frac{1}{2}}-\alpha_t\frac{ m_{i,t}}{\sqrt{\widehat{\upsilon}_{i,t}}}\big]
$\; 
}
}
\Output{resulting parameter $\bar{x}=\frac{1}{n}\sum_{i=1}^{n}x_{i,T+1}$}
\end{algorithm}\DecMargin{1.5em}

It is worth mentioning that \textsc{Dadam} includes decentralized  variants of many well-known algorithms as special cases, including \textsc{Adagrad}, \textsc{Rmsprop}, \textsc{Amsgrad}, \textsc{Sgd} and \textsc{Sgd} with momentum. We also note that \textsc{Dadam} computes adaptive learning rates from  estimates of both first and second moments of the gradients similar to \textsc{Amsgrad}. However, \textsc{Dadam} uses a larger learning rate in comparison to \textsc{Amsgrad} and yet incorporates the intuition of slowly decaying the effect of previous gradients on the learning rate. In particular, we design the update expression of the second moment estimate of the gradient to be
\begin{equation}\label{eq:relax:move}
\widehat{\upsilon}_{i,t}=\beta_{3}  \widehat{\upsilon}_{i,t-1}+(1-\beta_{3})\max(\widehat{\upsilon}_{i,t-1},\upsilon_{i,t}),
\end{equation}
where $\beta_{3} \in [0,1)$. The decay parameter $\beta_{3}$ is an important component of the \textsc{Dadam} framework, since it enables us to develop a convergent adaptive method similar to \textsc{Amsgrad} ($\beta_3= 0$), while maintaining the efficiency of \textsc{Adam} ($\beta_3= 1$).

Next, we introduce the measure of regret for assessing the performance of \textsc{Dadam} against a sequence of successive minimizers. In the framework of online convex optimization, the performance of algorithms is assessed by regret that measures how competitive the algorithm is with respect to the best fixed solution \cite{mateos2014distributed,hazan2016introduction}. However, the notion of regret fails to illustrate the performance of online algorithms in a dynamic setting. To overcome this issue, we consider a more stringent metric, the dynamic regret \cite{hall2015online, besbes2015non, zinkevich2003online}, in which the cumulative loss of the learner is compared against the minimizer sequence $\{x^{*}_t\}_{t=1}^{T}$, i.e.,
\begin{equation*}\label{rt}
  {\bf Reg}^C_T:= \frac{1}{n}\sum_{i=1}^n\sum_{t=1}^T f_{i,t}(x_{i,t})-\sum_{t=1}^T f_t(x^*_t),
\end{equation*}
where $x^*_t =  \argmin_{x \in \mathcal{X}} f_t(x)$.

On the other hand, in the framework of nonconvex optimization, it is common to state convergence guarantees of an algorithm towards a $\zeta$-approximate stationary point; that is, there exists some iterate $x_{i,t}$ for which $ \|\nabla f_t(x_{i,t})\| \leq \zeta$. Influenced by \cite{hazan2017efficient}, we provide the definition of projected gradient and introduce \textit{local regret} next, a new notion of regret which quantifies the moving average of gradients over a network.
\begin{defn}\label{def:G} {\normalfont (Local Regret)}. Assume $f_i:\mathcal{X}\rightarrow \mathbb{R}$ is a differentiable function on a closed convex set $\mathcal{X}\subseteq \mathbb{R}^p$. Given a step-size $\alpha>0$, we define $G_{\mathcal{X}}(x,f_i,\alpha): \mathcal{X}\rightarrow \mathbb{R}^p$ the projected gradient of $f_i$ at $x$, by
\begin{equation}\label{333}
 G_{\mathcal{X}}(x,f_i,\alpha)=\frac{\sqrt{\widehat{\upsilon}_{i}}}{\alpha}(x-x_i^+), \qquad \forall i\in\mathcal{V},
\end{equation}
with
\begin{equation}\label{531}
x_i^+= \textnormal{argmin}_{y\in \mathcal{X}}\{\langle y,\frac{m_{i}}{\sqrt{\widehat{\upsilon}_{i}}}\rangle+\frac{1}{2\alpha}\|y-\sum_{j=1}^n[W]_{ij}x_{j}\|^2\},
\end{equation}
where $m_{i}$ and $\widehat{\upsilon}_{i}$ are defined in Algorithm~\ref{alg}. Then, the local regret of an online algorithm is given by
\begin{equation*}
{\bf Reg}^N_T:=\frac{1}{n}\sum_{i=1}^n\min_{t \in [T]}  \|G_{\mathcal{X}}(x_{i,t},\bar{f}_{i,t},\alpha_t)\|^2,
\end{equation*}
where $\bar{f}_{i,t}(x_{i,t})=\frac{1}{t}\sum_{s=1}^tf_{i,s}(x_{i,t})$ is an aggregate loss.
\end{defn}

We analyze the convergence of \textsc{Dadam} as applied to minimization problem \eqref{125} using regrets ${\bf Reg}^C_T$ and ${\bf Reg}^N_T$. Note that \textsc{Dadam} is initialized at $x_{i,1} = 0$ to keep the presentation of the convergence analysis clear. In general, any initialization can be selected for implementation purposes.

\section{Convergence Analysis}\label{sec:conv}

Next, we establish convergence properties of \textsc{Dadam} under the following assumptions:
\begin{assum}\label{2040}
The weighted adjacency matrix $W$ of graph $\vphantom{\sum^N} \mathcal{G} = (\mathcal{V}, \mathcal{E})$ is doubly stochastic with a positive diagonal.  Specifically, the information received from agent $j\not=i$, $\vphantom{\sum^N} [W]_{ij}$ satisfies
\begin{equation}\label{W}
  \sum_{i=1}^n[W]_{ij}=\sum_{j=1}^n[W]_{ij}=1,    \qquad [W]_{jj} >0.
\end{equation}
 \end{assum}
\begin{assum}\label{2020}
For all $i \in \mathcal{V}$ and $t \in [T]$, the function $f_{i,t}(x)$ is differentiable over $\mathcal{X}$, and has Lipschitz continuous gradient on this set, i.e., there exists $\rho < \infty $ so that
\begin{equation*}
  \|\nabla f_{i,t}(x)-\nabla f_{i,t}(y)\|\leq \rho \|x-y\|, \qquad  \forall x,y \in \mathcal{X}.
\end{equation*}
Further, there exists $L < \infty $ such that
\begin{equation} \label{as1}
  |f_{i,t}(x)-f_{i,t}(y)|\leq L \|x-y\|, \qquad  \forall x,y \in \mathcal{X}.
\end{equation}
\end{assum}
%

\begin{assum}\label{2030}
For all $i \in \mathcal{V}$ and $t \in [T]$,  the  stochastic gradient denoted by ${\boldsymbol g}={\boldsymbol \nabla}f_{i,t}(x_{i,t})$, satisfies
\begin{align*}\label{condition}
\, \,&\e{\vphantom{\norm{{\boldsymbol \nabla}f_{i,t}(x_{i,t})}_*^2}{\boldsymbol \nabla}f_{i,t}(x_{i,t})\big\vert \F_{t-1}}=\nabla f_{i,t}(x_{i,t}),\\& \e{\norm{{\boldsymbol \nabla}f_{i,t}(x_{i,t})}^2\big\vert \F_{t-1}} \leq \xi^2,
\end{align*}
where $\F_t$ is the $\xi$-field containing all information prior to the onset of round $t+1$.
\end{assum}

\subsection{Convex Case}

Next, we focus on the case where for all $i \in \mathcal{V}$ and $t \in \{1,\ldots,T\}$, the agent $i$ at time $t$ has access to the exact gradient $g_{i,t}=\nabla f_{i,t}(x_{i,t})$.

Theorems \ref{regret} and \ref{theorem2} characterize the hardness of the problem via a complexly measure that captures the pattern of the minimizer sequence $\{x^{*}_t\}_{t=1}^{T}$, where $x^*_t =  \argmin_{x \in \mathcal{X}} f_t(x)$. Subsequently, we would like to provide a regret bound in terms of
\begin{equation}\label{equ:dynamic:regret}
D_{T,d}= \sum_{t=1}^{T-1} |x^*_{t+1,d}-x^*_{t,d}| \qquad \text{for} \qquad  d\in\{1,...,p\},
\end{equation}
which represents the variations in $\{x^{*}_t\}_{t=1}^{T}$.

Further, the following theorems establish a tight connection between the convergence rate of distributed adaptive methods and the spectral properties of the underlying network. The inverse dependence on the spectral gap $1-\sigma_2(W)$ is quite natural and for many families of undirected graph, we can give order-accurate estimate on $1-\sigma_2(W)$ [\cite{nedic2018network},~Proposition 5], which translate into estimates of convergence time.
%

\begin{thm}\label{regret}
Suppose Assumption~\ref{2040} holds and the parameters $\beta_1,\beta_2 \in [0,1)$ satisfy $\eta =\frac{\beta_1}{\sqrt{\beta_2}}<1$. Let $\beta_{1,t}=\beta_1\lambda^{t-1}, \lambda\in(0,1)$ and $\|\nabla f_{i,t}(x_t) \|_{\infty}\leq G_{\infty}$  for all $i \in \mathcal{V}$ and $t\in\{1,\ldots,T\}$. Then, using a step-size $\alpha_t=\frac{\alpha}{\sqrt{t}}$ for the sequence $x_{i,t}$ generated by Algorithm~\ref{alg}, we have
\begin{align*}
{\bf Reg}^C_T &\leq\frac{\alpha\sqrt{1+\log T}}{2\sqrt{n}\sqrt{(1-\beta_2)(1-\beta_3)}}\sum_{d=1}^p\|g_{1:T,d}\|
\\&+\sum_{d=1}^p\frac{G_{\infty}\gamma_{\infty}(1+{\gamma_{\infty}}/{(2\alpha)})}{(1-\beta_1)^2(1-\lambda)^2}
+\sum_{d=1}^{p}\frac{\gamma_{\infty}(\gamma_{\infty}+D_{T,d})}{\sqrt{n}(1-\beta_1)\alpha}\sqrt{T\widehat{\upsilon}_{T,d}}\\&+  \frac{4\alpha\sqrt{1+\log T}\sum_{d=1}^{p}\|g_{1:T,d} \|}{(1-\sigma_2(W))\sqrt{(1-\beta_1)}\sqrt{(1-\eta)}\sqrt{(1-\beta_2)(1-\beta_3)}}.
\end{align*}
\end{thm}

Next, we analyze the stochastic convex setting and extend the result of Theorem \ref{regret} to the noisy case where agents have access to stochastic gradients of the objective function \eqref{125}.

\begin{thm}\label{theorem2}
Suppose Assumptions \ref{2040} and \ref{2030} hold. Further, the parameters $\beta_1,\beta_2 \in [0,1)$ satisfy $\eta =\frac{\beta_1}{\sqrt{\beta_2}}<1$. Let $\beta_{1,t}=\beta_1\lambda^{t-1}, \lambda\in(0,1)$. Then, using a step-size $\alpha_t=\frac{\alpha}{\sqrt{t}}$ for the sequence $x_{i,t}$ generated by Algorithm~\ref{alg}, we have
\begin{align*}
\e{{\bf Reg}^C_T} &\leq \frac{\alpha\sqrt{1+\log T}}{2\sqrt{n}\sqrt{(1-\beta_2)(1-\beta_3)}}\sum_{d=1}^p\e{\|{\boldsymbol g}_{1:T,d}\|}
\\&+\sum_{d=1}^p\frac{\xi\gamma_{\infty}(1+{\gamma_{\infty}}/{(2\alpha)})}{(1-\beta_1)^2(1-\lambda)^2}
+ \sum_{d=1}^{p}\frac{\gamma_{\infty}(\gamma_{\infty}+D_{T,d})}{\sqrt{n}(1-\beta_1)\alpha}\sqrt{T}\e{\sqrt{\widehat{\upsilon}_{T,d}}}
\\&+  \frac{4\alpha\sqrt{1+\log T}\sum_{d=1}^{p}\e{\|{\boldsymbol g}_{1:T,d} \|}}{(1-\sigma_2(W))\sqrt{(1-\beta_1)}\sqrt{(1-\eta)}\sqrt{(1-\beta_2)(1-\beta_3)}}.
\end{align*}
\end{thm}

%
%

\begin{rem}
Theorems~\ref{regret} and \ref{theorem2} suggest that, similar to adaptive algorithms such as \textsc{Adam}, \textsc{\textsc{Adagrad}} and \textsc{Amsgrad}, the summation terms in the regret bound can be much smaller than their upper bounds when $\sum_{d=1}^p \| g_{1:T,d}\|\leq pG_{\infty} \sqrt{T}$ and $\sum_{d=1}^p \sqrt{T\widehat{\upsilon}_{T,d}}\leq p G_{\infty}\sqrt{T}$. Thus, the regret bound of \textsc{Dadam} can be considerably better than the ones provided by standard mirror descent and gradient descent methods in both centralized \cite{zinkevich2003online, hall2015online, besbes2015non} and decentralized \cite{duchi2012dual,jiang2017collaborative,shahrampour2018distributed, lian2017can} settings.
\end{rem}

\subsection{Nonconvex Case}\label{se4}

In this section, we provide convergence guarantees for \textsc{Dadam} for the nonconvex minimization problem \eqref{125} defined over a closed convex set $\mathcal{X}$. To do so, we use the projection map $\Pi_{\mathcal{X}} $ instead of $\Pi_{\mathcal{X}, \sqrt{\text{diag}(\widehat{\upsilon}_{i,t})}}$ for updating parameters $x_{i,t}$ for all $t\in\{1, \dots, T\}$ and $i\in \mathcal{V}$ (see, Algorithm~\ref{alg} for details).

To analyze the convergence of \textsc{Dadam} in the nonconvex setting, we assume $ g_{i,1,d}>0$ for all $i\in \mathcal{V}$  and  $d\in \{1, \dots, p\}$. This assumption is usually needed for numerical stability \footnote{If $g_{i,1,d} = 0$ for some $i$ and $d$, division by 0 may occur at $t=1$.} and similar assumptions are also widely used to establish the convergence of adaptive methods in the nonconvex setting  \cite{chen2018convergence,ward2018adagrad,basu2018convergence}.  In addition, similar to the convex setting, we let $\|\nabla f_{i,t}(x_t) \|_{\infty}\leq G_{\infty}$ for all $i\in \mathcal{V}$ and $t \in\{1,\dots,T\}$. These two assumptions together with the update rule of $\widehat{\upsilon}_{i,t}$ defined in \eqref{eq:relax:move} imply
\begin{align}\label{v}
\underline{\upsilon} \leq \sqrt{\widehat{\upsilon}_{i,t,d}} \leq \bar{\upsilon}, \qquad \text{~~$i\in \mathcal{V},$ ~~ $t\in\{1, \dots, T\}$, ~~ $d\in \{1, \dots, p\}$},
\end{align}
where $\underline{\upsilon}$ and $\bar{\upsilon}$ are positive constants.


The following theorem establishes the convergence rate of decentralized adaptive methods in the nonconvex setting.

\begin{thm}\label{C}
Suppose Assumptions \ref{2040} and \ref{2020} hold.
Further, the parameters $\beta_1,\beta_2 \in [0,1)$ satisfy $\eta= \frac{\beta_1}{\sqrt{\beta_2}}<1$. Let $\beta_{1,t}=\beta_1\lambda^{t-1}, \lambda\in(0,1)$. Choose the positive sequence $\{\alpha_t\}_{t=1}^{T}$ such that $0< \alpha_t \leq \frac{(2-\beta_1)\underline{\upsilon}^2}{\rho\bar{\upsilon}}$ with $ \alpha_t < \frac{(2-\beta_1)\underline{\upsilon}^2}{\rho\bar{\upsilon}}$ for at least one $t$. Then, for the sequence $x_{i,t}$ generated by Algorithm~\ref{alg}, we have
\begin{align}\label{nooon}
{\bf Reg}^N_T \nonumber\leq\frac{1}{\vartheta_t}&[(2+\log T)2L\max_{t\in \{2,\ldots,T\}} \frac{2\sqrt{n}}{(1-\eta)\sqrt{(1-\beta_2)(1-\beta_3)}}\sum_{s=0}^{t-1}\alpha_s\sigma_2^{t-s-1}(W) \\&+\sum_{t=1}^{T}\frac{\alpha_t\beta_{1,t}\bar{\upsilon}}{2(1-\beta_{1,t})(1-\eta)^2(1-\beta_2)}],
\end{align}
where  $\vartheta_t= \sum_{t=1}^T [ \frac{(2-\beta_1)\alpha_t}{2\bar{\upsilon}}-\frac{\rho \alpha_t^2}{2\underline{\upsilon}^2}]$.
\end{thm}

The following corollary shows that \textsc{Dadam} using a certain step-size leads to a near optimal regret bound for nonconvex functions.

\begin{cor}\label{Corollary3}
Under the same conditions of Theorem~\ref{C}, using the step-sizes
$\alpha_t = \frac{(2-\beta_1)\underline{\upsilon}^2}{2\rho\bar{\upsilon}}$ and $\beta_{1,t}=\beta_1\lambda^{t-1}, \lambda\in(0,1)$ for all $t\in\{1,\ldots,T\}$, we have
\begin{align}\label{zs}
\nonumber{\bf Reg}^N_T &\leq \big(\frac{2\bar{\upsilon}^2}{(2-\beta_1)(1-\beta_{1})(1-\eta)^2(1-\beta_2)(1-\lambda)}\big) \frac{1}{T} \\&+\big(\frac{16\sqrt{n}\bar{\upsilon}L }{(2-\beta_1)(1-\eta)\sqrt{(1-\beta_2)(1-\beta_3)}(1-\sigma_2(W))}\big) \frac{(2+\log T)}{T} .
\end{align}
\end{cor}

To complete the analysis of our algorithm in the nonconvex setting, we provide the regret bound for \textsc{Dadam}, when stochastic gradients are accessible to the learner.

\begin{thm}\label{theorem3}
Suppose Assumptions \ref{2040}-\ref{2030} hold.
Further, the parameters $\beta_1,\beta_2 \in [0,1)$ satisfy $\eta= \frac{\beta_1}{\sqrt{\beta_2}}<1$. Let $\beta_{1,t}=\beta_1\lambda^{t-1}, \lambda\in(0,1)$. Choose the positive sequence $\{\alpha_t\}_{t=1}^{T}$ such that $0< \alpha_t \leq \frac{(2-\beta_1)\underline{\upsilon}^2}{\rho\bar{\upsilon}}$ with $ \alpha_t < \frac{(2-\beta_1)\underline{\upsilon}^2}{\rho\bar{\upsilon}}$ for at least one $t$. Then, for the sequence $x_{i,t}$ generated by Algorithm~\ref{alg}, we have
\begin{align}\label{noon}
\e{{\bf Reg}^N_T} \nonumber&\leq\frac{1}{\vartheta_t}[(2+\log T)2L\max_{t\in \{2,\ldots,T\}} \frac{2\sqrt{n}}{(1-\eta)\sqrt{(1-\beta_2)(1-\beta_3)}}\sum_{s=0}^{t-1}\alpha_s\sigma_2^{t-s-1}(W) \\&+\sum_{t=1}^{T}\frac{\alpha_t\beta_{1,t}\bar{\upsilon}}{2(1-\beta_{1,t})(1-\eta)^2(1-\beta_2)}+\frac{\bar{\upsilon}\xi^2}{\underline{\upsilon}^2(1-\beta_1)}\sum_{t=1}^T \alpha_t],
\end{align}
where  $\vartheta_t= \sum_{t=1}^T [ \frac{(2-\beta_1)\alpha_t}{2\bar{\upsilon}}-\frac{\rho \alpha_t^2}{2\underline{\upsilon}^2}]$.
\end{thm}

\subsubsection{When does \textsc{Dadam} Outperform \textsc{Adam}?}

We next theoretically justify the potential advantage of the proposed decentralized algorithm \textsc{Dadam} over centralized adaptive moment estimation methods such as \textsc{Adam}. More specifically, the following corollary shows that when $T$ is sufficiently large, the $\frac{1}{T}$ term will be dominated by the $\frac{1}{\sqrt{nT}}$ term which leads to a $\frac{1}{\sqrt{nT}}$ convergence rate.

\begin{cor}\label{Corollary10}
Suppose Assumptions \ref{2040}-\ref{2030} hold.  Moreover, the parameters $\beta_1,\beta_2\in [0,1)$ satisfy $\eta= \frac{\beta_1}{\sqrt{\beta_2}}<1$ and $\beta_{1,t}=\beta_1\lambda^{t-1}, \lambda\in(0,1)$.  Choose the step-size sequence as $\alpha_t=\frac{\alpha}{\sqrt{nT}}$ with $ \alpha=\frac{(2-\beta_1)\underline{\upsilon}^2}{\rho\bar{\upsilon}}$. Then, for the sequence $x_{i,t}$ generated by Algorithm~\ref{alg}, we have
\begin{equation}\label{non}
\frac{\e{{\bf Reg}^N_T}}{T}\leq \big(\frac{8\bar{\upsilon}\alpha}{\underline{\upsilon}^2(1-\beta_1)}\big)\frac{\xi^2}{\sqrt{nT}}
+2\big( f_1(x_1)- f_1(x^*_1)\big) \frac{1}{T},
\end{equation}
if the total number of time steps $T$ satisfies
\begin{subequations}
\begin{align}
T&\label{ag}\geq(I_1+I_2),
\\
T&\geq\max\{\frac{4\rho^2\bar{\upsilon}^2}{n\underline{\upsilon}^4(2-\beta_1)^2},\frac{4\bar{\upsilon}^2n}{(2-\beta_1)^2}\} \label{ag1},
\end{align}
\end{subequations}
where
\begin{align*}\label{99}
  &I_1=\frac{\underline{\upsilon}^2}{2(1-\eta)^2(1-\beta_2)(1-\lambda)\xi^2}, \\&I_2=\frac{2\sqrt{n} L \underline{\upsilon}^2(1-\beta_1) }{(1-\eta)\sqrt{(1-\beta_2)(1-\beta_3)}(1-\sigma_2(W)) \bar{\upsilon} \xi^2}.
\end{align*}
\end{cor}

Let $\varsigma$-approximation solution of \eqref{125} be defined by $\frac{\e{{\bf Reg}^N_T}}{T}\leq \varsigma$. Corollary~\ref{Corollary10} indicates that the total computational complexity of \textsc{Dadam} to achieve an $\varsigma$-approximation solution is bounded by $O (\frac{1}{\varsigma^2})$.

\section{An Extension of \textsc{Dadam} with a Corrected Update Rule} \label{sec:ex}

Compared to classical centralized algorithms, decentralized algorithms encounter more restrictive assumptions and typically worse convergence rates. Recently, for time-invariant
graphs, \cite{shi2015extra} introduced a corrected decentralized gradient method in order to cancel the steady state error in decentralized gradient descent and provided a linear rate of convergence if the objective function is strongly convex. Analogous convergence results are given in \cite{nedic2017achieving} even for the case of time-variant graphs. Similar to \cite{shi2015extra,nedic2017achieving}, we provide next
a corrected update rule for adaptive methods, given by
\begin{equation} \label{corr_dADAM}
\begin{array}{cl}
x_{i,t+1}^{\textsc{C-Dadam}} = x_{i,t+1}+\underbrace{\sum\limits_{s=0}^{t-1} \sum_{j=1}^n[W-\widehat{W}]_{ij}x_{j,s}}_{\text{correction}},
\end{array}
\end{equation}
for all $i\in \mathcal{V}$, and $t \in \{1,\ldots,T\}$, where $x_{i,t+1} $ is generated by Algorithm~\ref{alg} and $\widehat{W}=\frac{I+W}{2}$.


We note that a \textsc{C-Dadam} update is a \textsc{Dadam} update with a cumulative correction term. The summation in \eqref{corr_dADAM} is necessary, since each individual term $\sum_{j=1}^n[W-\widehat{W}]_{ij}x_{j,s}$ is asymptotically vanishing and the terms must work cumulatively \cite{shi2015extra}.

\section{Numerical Results}\label{sec4}

In this section, we evaluate the effectiveness of the proposed \textsc{Dadam}-type algorithms such as \textsc{D\textsc{Adagrad}}, \textsc{Dadadelta}, \textsc{Drmsprop}, and \textsc{Dadam} by comparing them with \textsc{Sgd} \cite{robbins1985stochastic}, \textsc{Dsgd} \cite{nedic2009distributed, jiang2017collaborative, lian2017can,shahrampour2018distributed} and corrected \textsc{Dsgd} (\textsc{C-Dsgd}) \cite{shi2015extra,tang2018d}.

The corrected variants of proposed \textsc{Dadam}-type algorithms are denoted by  \textsc{C-D\textsc{Adagrad}}, \textsc{C-Dadadelta}, \textsc{C-Drmsprop}, and \textsc{C-Dadam}. We also note that if the mixing matrix $W$ in Algorithm~\ref{alg} is chosen the $n \times n$ identity matrix, then above algorithms reduce to the centralized adaptive methods. These algorithms are implemented with their default settings\footnote{https://keras.io/optimizers/}.

All algorithms have been run on a Mac machine equipped with a 1.8 GHz Intel Core i5 processor and 8 GB 1600 MHz DDR3. Code to reproduce experiments is to be found at
\texttt{https://github.com/Tarzanagh/DADAM}.

In our experiments, we
use the Metropolis constant edge weight matrix $W$ \cite{boyd2004fastest} (see, Section~\ref{sec:matrices} for details). The connected network is randomly generated with $n=10$ agents and connectivity ratio $r=0.5$.


Next, we mainly focus on the convergence rate of algorithms instead of the running time. This is because the implementation of \textsc{Dadam}-type algorithms is a minor change over the standard decentralized stochastic algorithms such as \textsc{Dsgd} and \textsc{C-Dsgd}, and thus they have almost the same running time to finish one epoch of training, and both are faster than the centralized stochastic algorithms such as \textsc{Adam} and \textsc{Sgd}. We note that with high network latency, if a decentralized algorithm (\textsc{Dadam} or \textsc{Dsgd}) converges with a similar running time as the centralized algorithm, it can be up to one order of magnitude faster \cite{tang2018d}. However, the convergence rate depending on the ``adaptiveness" is different for both algorithms.

\subsection{Regularized Finite-sum Minimization Problem}\label{sec:binaryclass}

Consider the following online distributed learning setting: at each time $t$, $b_i$ randomly generated data points are given to every agent $i$ in the form of $(\bm y_{t,i,j}, \bm z_{t,i,j})$. Our goal is to learn the model parameter $x \in \mathbb{R}^p$ by solving the $\ell_2$ regularized finite-sum minimization problem \eqref{125} with
\begin{align}\label{eq:loglossfun}
 f_{i,t}(x)&=\frac{1}{b_i}\sum_{j=1}^{b_i} L(x, \bm y_{t,i,j}, \bm z_{t,i,j}) + \nu \|x\|_2^2,
\end{align}
where $L(x, \bm y_{t,i,j}, \bm z_{t,i,j})$ is the loss function, and $\nu$ is the regularization parameter.

For $\mathcal{X}$, we consider the $\ell_1$ ball $\mathcal{X}_{\ell_1} = \{ x \in \mathbb{R}^p : \| x \|_1 \leq r \},$ when a sparse classifier is preferred.

From Theorem~\ref{regret}, we would choose a constant step-size $\alpha_t=\alpha=\sqrt{1-\sigma_2(W)}$ and diminishing step-sizes $\alpha_t=\sqrt{\frac{1-\sigma_2(W)}{t}}$, for $t\in\{1,\dots, T\}$ in order to evaluate the adaptive strategies. All other parameters of the algorithms and problems are set as follows: $\beta_1 =\beta_3 = 0.9$, and $\beta_2=0.999$; the mini-batch size is set to 10,  the regularization parameter $\nu =0.1$ and the dimension of model parameter $p=100$.

The numerical results are illustrated in Figure
\ref{fig:appLogistic} for the synthetic datasets. It can be seen that the distributed adaptive algorithms significantly outperform \textsc{Dsgd} and its corrected variants.

\subsection{Neural Networks} \label{result:stat}

Next, we present the experimental results using the \texttt{MNIST} digit recognition task. The model for training a simple multilayer perceptron (MLP) on the \texttt{MNIST} dataset was taken from Keras.GitHub \footnote{\texttt{https://github.com/keras-team/keras}}. In our implementation, the model function has 15 dense layers of size 64. Small $\ell_2$ regularization with regularization parameter 0.00001 is added to the weights of the network and the mini-batch size is set to 32.

We compare the accuracy of \textsc{Dadam} with that of the \textsc{Dsgd} and the Federated Averaging (FedAvg) algorithm \cite{mcmahan2016communication} which also performs data parallelization without decentralized computation. The parameters for \textsc{Dadam} is selected in a way similar to the previous experiments.  In our implementation, we use same number of agents and choose $E =C=1$ as the parameters in the FedAvg algorithm since it is close to a connected topology scenario as considered in the \textsc{Dadam} and \textsc{Adam}.    It can be easily seen from Figure \ref{fig:logloss} that \textsc{Dadam} can achieve high accuracy in comparison with the \textsc{Dsgd} and FedAvg.

\section{Conclusion}\label{sec5}
A decentralized adaptive algorithm was proposed for distributed gradient-based optimization of online and stochastic objective functions. Convergence properties of the proposed algorithm were established for convex and nonconvex functions in both stochastic and deterministic settings. Numerical results on some synthetics and real datasets show the efficiency and effectiveness of the proposed method in practice.

\onecolumn
\begin{figure}[h]
	\centering \makebox[0in]{
\begin{tabular}{c}
\includegraphics[scale=0.13]{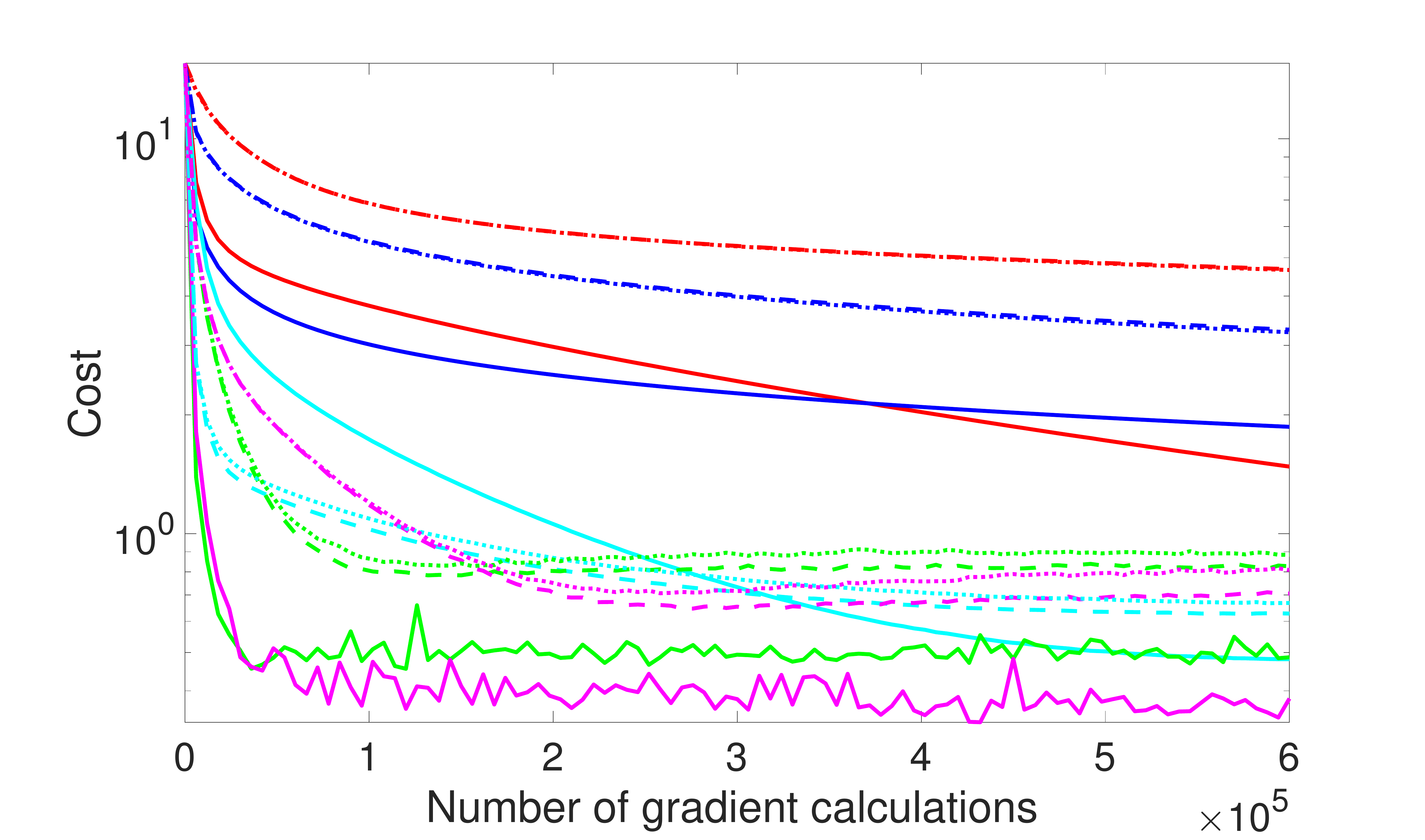}
\includegraphics[scale=0.13]{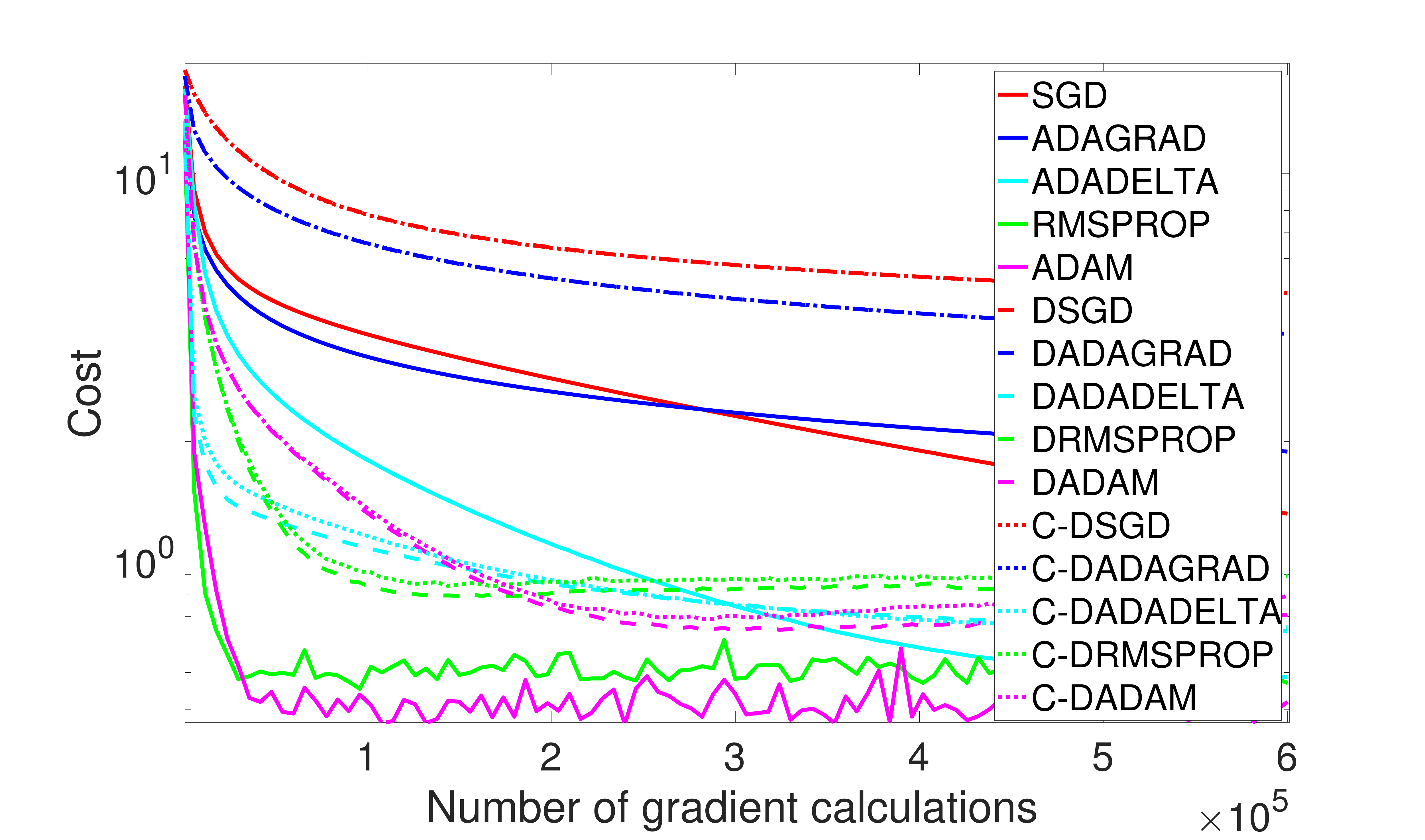}
\\
(a) ~$\ell_2$-regularized softmax regression problem, \texttt{MNIST} dataset.
\\
\includegraphics[scale=0.13]{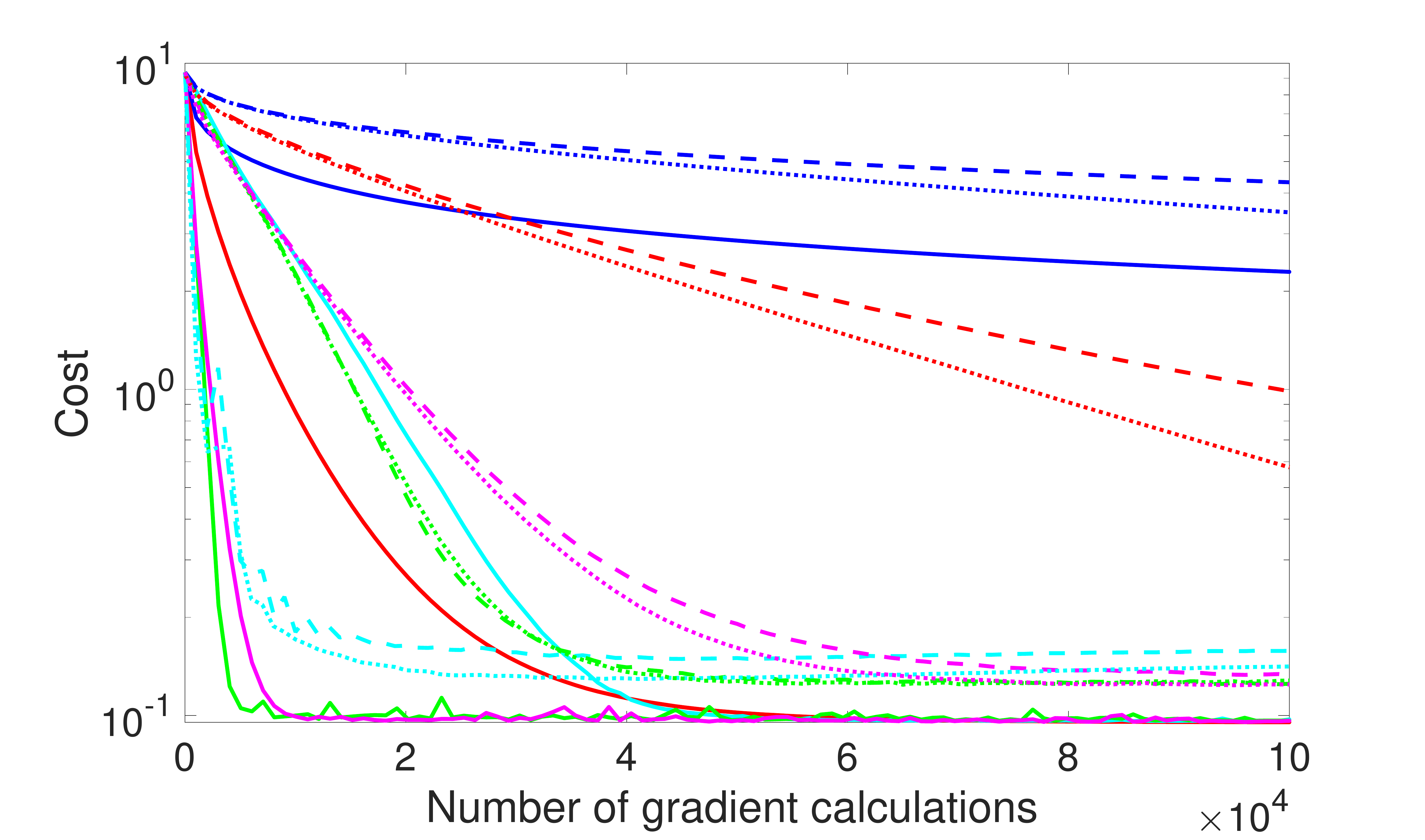}
\includegraphics[scale=0.13]{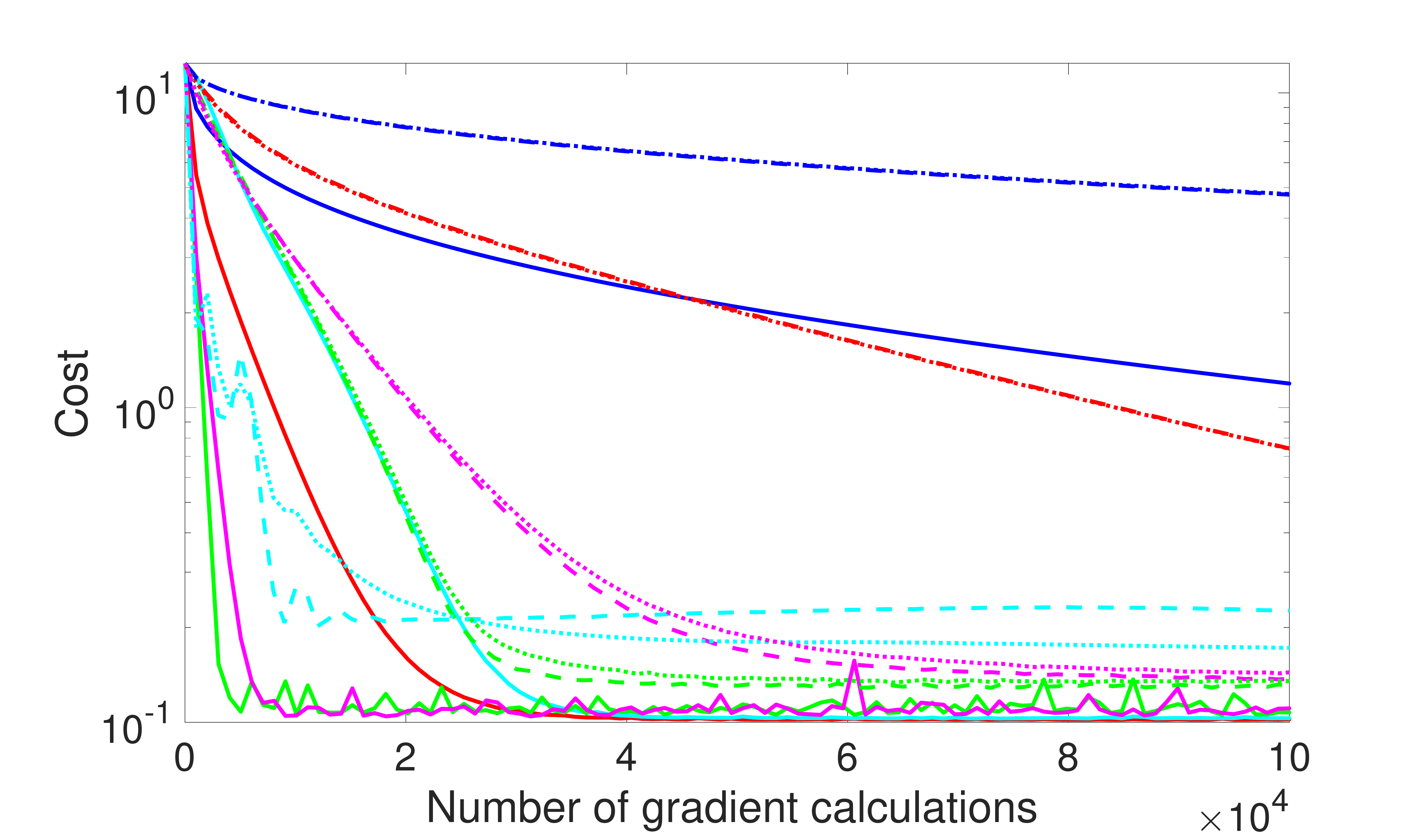}
\\
(b)~$\ell_2$-regularized support vector machine (SVM) problem, \texttt{Mushroom} dataset. \\
\includegraphics[scale=0.13]{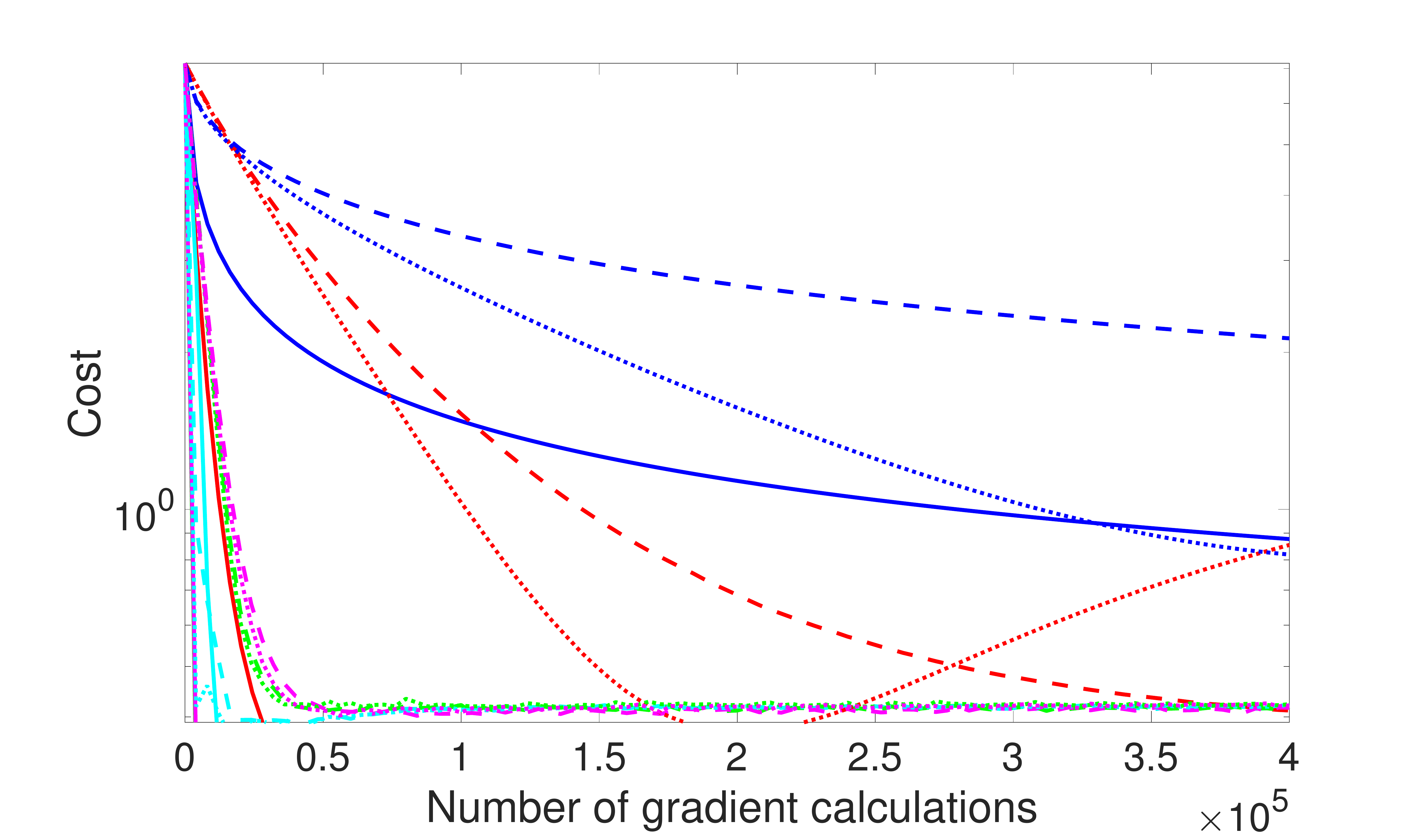}
\includegraphics[scale=0.13]{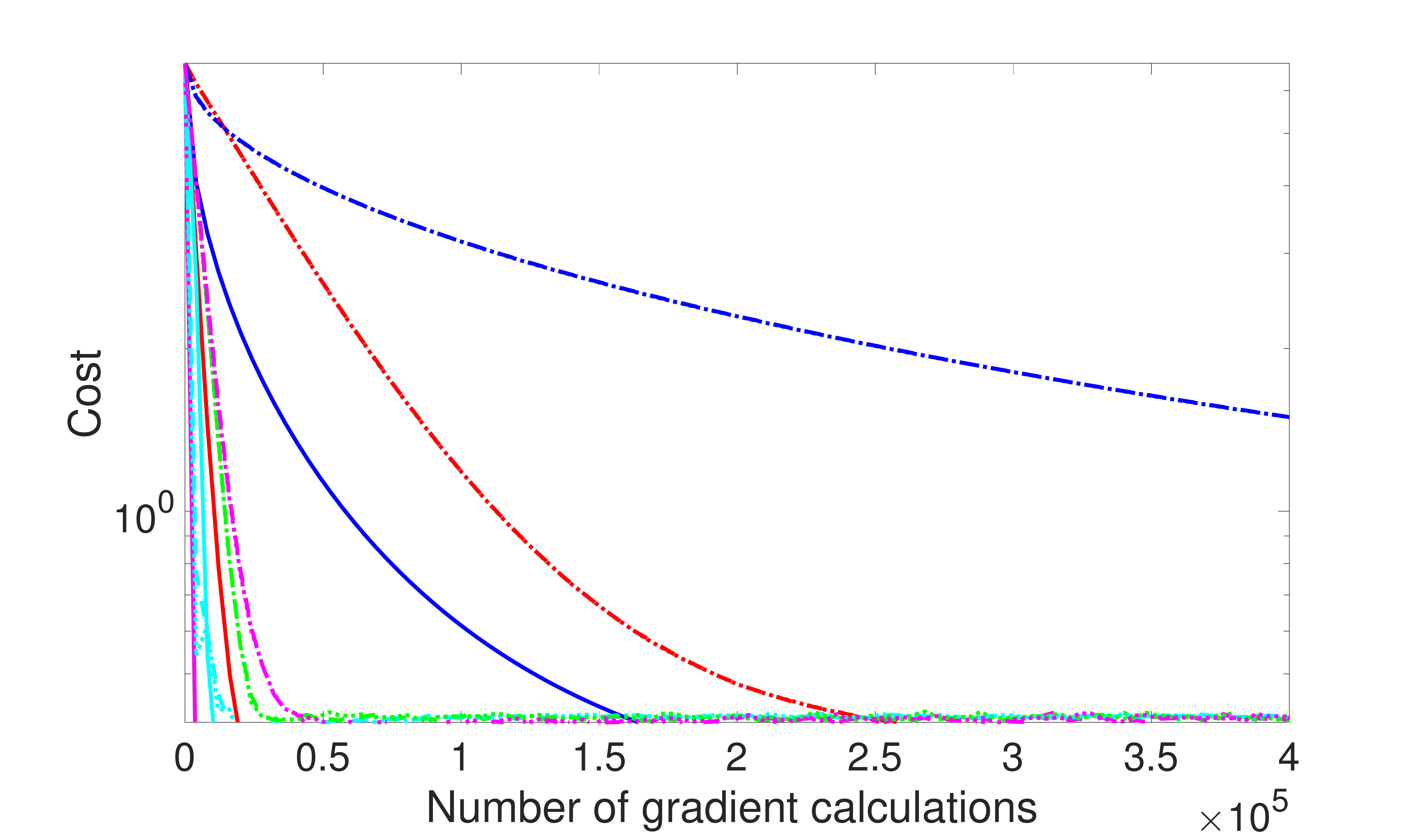}
\\
(c) ~$\ell_1-\ell_2$-regularized logistic regression problem, \texttt{data-100d-10000} dataset.\\
\end{tabular}
}
\caption{Convergence of different stochastic algorithms over 100 epochs on the  datasets from SGDLibrary \cite{kasai2018sgdlibrary}. (\textbf{left}) fixed step-size and (\textbf{right}) diminishing step-size.The legend for all curves is on the top right.}\label{fig:appLogistic}
\end{figure}

\begin{figure}[h]
	\centering \makebox[0in]{
\begin{tabular}{cc}
\includegraphics[scale=0.435]{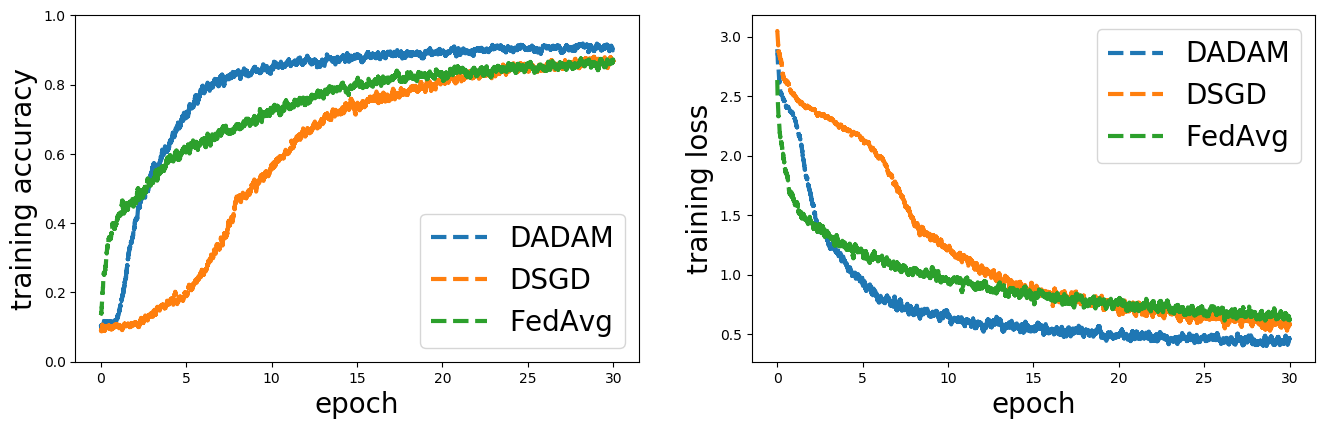}
\end{tabular}}
\caption{ Training simple MLP on the \texttt{MNIST} digit recognition dataset. Training loss and accuracy of different distributed algorithms over 30 epochs.}\label{fig:logloss}
\end{figure}

\section*{Acknowledgements}
The second author is grateful for a discussion of decentralized methods with Davood Hajinezhad.

\bibliography{ref}
\bibliographystyle{ieeetr}

\section{Supplementary Material}
Next, we establish a series of lemmas used in the proof of main theorems.

\subsection{Auxiliary Lemmas}

\begin{lem}\label{k} \cite{beck2003mirror}
Let $\mathcal{X}$ be a nonempty closed convex set in $\mathbb{R}^p$. Then, for any  $d \in \mathcal{X}$, we have
\begin{equation*}
   \langle x^*-d,a\rangle \leq \frac{1}{2}\|d-c\|^2-\frac{1}{2}\|d-x^*\|^2-\frac{1}{2}\|x^*-c\|^2,
\end{equation*}
where
\begin{equation*}
  x^*=\argmin_{x\in \mathcal{X}}\{\langle  a,x\rangle+\frac{1}{2}\|x-c\|^2\}.
\end{equation*}
\end{lem}

\begin{lem}\label{pi} \cite{mcmahan2010adaptive}
For any $A\in \mathcal{S}^{p}_{+}$ and convex feasible set
$C\subset \mathbb{R}^p,$ suppose $a_1=\Pi_{C,A}[b_1], a_2=\Pi_{C,A}[b_2]$. Then, we have
$$ \|A^{\frac{1}{2}}(a_1-a_2) \| \leq \|A^{\frac{1}{2}}(b_1-b_2)\|.$$
\end{lem}

\begin{lem}\label{000}
Let $\beta_1,\beta_2 \in[0,1)$ satisfy $\eta = \frac{\beta_1}{\sqrt{\beta_2}}<1$. Then, for any $i\in\mathcal{V}$, we have
\begin{equation*}
\sum_{t=1}^{T}  \sum_{d=1}^{p} \frac{\alpha_t m^2_{i,t,d}}{\sqrt{\widehat{\upsilon}_{i,t,d}}}\leq \frac{\alpha\sqrt{1+\log T}}{(1-\beta_1)(1-\eta)\sqrt{(1-\beta_2)(1-\beta_3)}}\sum_{d=1}^{p}\|g_{i,1:T,d} \|,
\end{equation*}
where $\alpha_t= \frac{\alpha} {\sqrt{t}}$ for all $t\in\{1, \dots, T\}$.
\end{lem}

\begin{proof}
Using the update rule of moment vectors $m_{i,t}$ and $\widehat{\upsilon}_{i,t}$ in Algorithm~\ref{alg}, we have
\begin{align*}
  \sum_{t=1}^{T}\frac{\alpha m^2_{i,t,d}}{\sqrt{t\widehat{\upsilon}_{i,t,d}}} &= \sum_{t=1}^{T-1} \frac{\alpha m^2_{i,t,d}}{\sqrt{t \widehat{\upsilon}_{i,t,d}}}+ \frac{\alpha_T m^2_{i,T,d}}{\sqrt{(1-\beta_3)\max\{\widehat{\upsilon}_{i,T-1,d},\upsilon_{i,T,d}\}+\beta_3\widehat{\upsilon}_{i,T-1,d}}}\\
  &\leq \sum_{t=1}^{T-1} \frac{\alpha m^2_{i,t,d}}{\sqrt{t\widehat{\upsilon}_{i,t,d}}}+\frac{\alpha_T m^2_{i,T,d}}{\sqrt{(1-\beta_3)\upsilon_{i,T,d}}}\\
 &\stackrel{(i)}{=} \sum_{t=1}^{T-1} \frac{\alpha m^2_{i,t,d}}{\sqrt{t\widehat{\upsilon}_{i,t,d}}}+\frac{\alpha(\sum_{l=1}^T(1-\beta_1)\beta_1^{T-l}g_{i,l,d})^2}{\sqrt{(1-\beta_3)T\sum_{l=1}^T(1-\beta_2)\beta_2^{T-l}g^2_{i,l,d}}}\\
  &\stackrel{(ii)}{\leq} \sum_{t=1}^{T-1} \frac{\alpha m^2_{i,t,d}}{\sqrt{t\widehat{\upsilon}_{i,t,d}}}+\frac{\alpha}{\sqrt{T(1-\beta_2)(1-\beta_3)}}\frac{(\sum_{l=1}^T\beta_1^{T-l})(\sum_{l=1}^T\beta_1^{T-l}g^2_{i,l,d})}{\sqrt{\sum_{l=1}^T\beta_2^{T-l}g^2_{i,l,d}}}\\
   &\stackrel{(iii)}{\leq}  \sum_{t=1}^{T-1} \frac{\alpha m^2_{i,t,d}}{\sqrt{t\widehat{\upsilon}_{i,t,d}}}+\frac{\alpha}{(1-\beta_1)\sqrt{T(1-\beta_2)(1-\beta_3)}}\sum_{l=1}^T\frac{\beta_1^{T-l}g^2_{i,l,d}}{\sqrt{\beta_2^{T-l}g^2_{i,l,d}}}\\
   &{\leq}  \sum_{t=1}^{T-1} \frac{\alpha m^2_{i,t,d}}{\sqrt{t\widehat{\upsilon}_{i,t,d}}}+\frac{\alpha}{(1-\beta_1)\sqrt{T(1-\beta_2)(1-\beta_3)}}\sum_{l=1}^T \eta^{T-l}|g_{i,l,d}|,
   \end{align*}
where (i) follows from the fact that the update rules of $m_T$ and $\upsilon_T$ can be written as $m_T=(1-\beta_1)\sum_{l=1}^{T}\beta_1^{T-l}g_l$ and $\upsilon_T=(1-\beta_2)\sum_{l=1}^{T}\beta_2^{T-l}g_l^2$, respectively. (ii) follows from Cauchy-Schwarz inequality and the fact that $0\leq\beta_1<1$. Inequality~(iii) follows since $\sum_{l=1}^{T}\beta_1^{T-l}\leq\frac{1}{1-\beta_1}$.
Hence, we have
   \begin{align*}
    \sum_{t=1}^{T}\frac{\alpha m^2_{i,t,d}}{(1-\beta_1)\sqrt{t\widehat{\upsilon}_{i,t,d}}}
     &\leq  \sum_{t=1}^{T}\frac{\alpha}{(1-\beta_1)\sqrt{t(1-\beta_2)(1-\beta_3)}}\sum_{l=1}^t\eta^{t-l}|g_{i,l,d}|\\
   &= \frac{\alpha}{(1-\beta_1)\sqrt{(1-\beta_2)(1-\beta_3)}}\sum_{t=1}^{T}\frac{1}{\sqrt{t}}\sum_{l=1}^t\eta^{t-l}|g_{i,l,d}|\\
   &=  \frac{\alpha}{(1-\beta_1)\sqrt{(1-\beta_2)(1-\beta_3)}}\sum_{t=1}^{T}|g_{i,t,d}| \sum_{l=t}^T\frac{\eta^{l-t}}{\sqrt{l}}\\
    &{\leq}\frac{\alpha}{(1-\beta_1)\sqrt{(1-\beta_2)(1-\beta_3)}}\sum_{t=1}^{T}|g_{i,t,d}|\sum_{l=t}^T\frac{\eta^{l-t}}{\sqrt{t}}\\
     &\stackrel{(i)}{\leq} \frac{\alpha}{(1-\beta_1)\sqrt{(1-\beta_2)(1-\beta_3)}}\sum_{t=1}^{T}|g_{i,t,d}|\frac{1}{(1-\eta)\sqrt{t}}\\
    &\stackrel{(ii)}{\leq} \frac{\alpha}{(1-\beta_1)(1-\eta)\sqrt{(1-\beta_2)(1-\beta_3)}}\|g_{i,1:T,d} \|\sqrt{\sum_{t=1}^{T}\frac{1}{t}}\\
    &\stackrel{(iii)}{\leq} \frac{\alpha\sqrt{1+\log T}}{(1-\beta_1)(1-\eta)\sqrt{(1-\beta_2)(1-\beta_3)}}\|g_{i,1:T,d} \|,\\
   \end{align*}
 where inequality~(i) follows since
$\sum_{l=t}^T\eta^{l-t}\leq\frac{1}{(1-\eta)}$. Inequality~(ii) follows from Cauchy-Schwarz inequality. The inequality~(iii) follows since
\begin{equation}\label{log}
  \sum_{t=1}^{T}\frac{1}{t}\leq 1+\int_{t=1}^{T}\frac{1}{t} dt=1+\log t|_{1}^T =1+\log T.
\end{equation}
\end{proof}

Next, we provide an upper bound on the deviation of the local
estimates at each iteration from their consensual value. A
similar result has been proven in  \cite{shahrampour2017online} for online decentralized mirror descent; however, the following lemma
extends that of \cite{shahrampour2017online} to the online adaptive setting and takes into account the sparsity of gradient vector.

\begin{lem}[Network Error with Sparse Data] \label{50} Suppose Assumption~\ref{2040} holds. If $\beta_1,\beta_2\in[0,1)$ satisfy $\eta=\frac{\beta_1}{\sqrt{\beta_2}}<1$, then the sequence $x_{i,t}$ generated by Algorithm~\ref{alg} satisfies
  \begin{equation*}
    \sum_{t=1}^{T} \sum_{i=1}^{n}\frac{1}{\alpha_t} \|{\widehat{V}_{i,t}}^{\frac{1}{4}}(\bar{x}_{t}-x_{i,t} )\|^2\leq \frac{n\sqrt{n}\alpha\sqrt{1+\log T}\sum_{d=1}^{p}\|g_{1:T,d} \|}{(1-\sigma_2(W))^2(1-\beta_1)(1-\eta)\sqrt{(1-\beta_2)(1-\beta_3)}},
  \end{equation*}
where $\widehat{V}_{i,t} =\text{diag}(\widehat{\upsilon}_{i,t})$ and $\bar{x}_{t}= \frac{1}{n}\sum_{i=1}^nx_{i,t}.$
\end{lem}
\begin{proof}
Let $e_{i,t} := x_{i,t+1}-\sum_{j=1}^n[W]_{ij}x_{j,t}$, where $W$ satisfies \eqref{W}. Using the update rule of $x_{i,t+1}$ in Algorithm~\ref{alg}, we have
\begin{align}\label{208}
\sum_{t=1}^{T}\frac{1}{\alpha_t}\| \widehat{V}_{i,t}^{\frac{1}{4}}e_{i,t} \|^2\nonumber
\nonumber &= \sum_{t=1}^{T}\frac{1}{\alpha_t}\|\widehat{V}_{i,t}^{\frac{1}{4}}(x_{i,t+1}-\sum_{j=1}^n[W]_{ij}x_{j,t})\|^2  \\
  \nonumber &=  \sum_{t=1}^{T}\frac{1}{\alpha_t} \|\widehat{V}_{i,t}^{\frac{1}{4}}(\Pi_{\mathcal{X},\sqrt{\widehat{V}_{i,t}}}\big[\sum_{j=1}^n[W]_{ij}x_{j,t}-\frac{\alpha_t m_{i,t}}{\sqrt{\widehat{\upsilon}_{i,t}}}\big]-\sum_{j=1}^n[W]_{ij}x_{j,t} )\|^2\\
  \nonumber &\leq \sum_{t=1}^{T}\frac{1}{\alpha_t}\|\widehat{V}_{i,t}^{\frac{1}{4}}(\sum_{j=1}^n[W]_{ij}x_{j,t}-\alpha_t\widehat{V}_{i,t}^{\frac{-1}{2}} m_{i,t}-\sum_{j=1}^n[W]_{ij}x_{j,t})\|^2 \\
  &= \sum_{t=1}^{T} \sum_{d=1}^{p}\frac{\alpha_t m^2_{i,t,d}}{\sqrt{\widehat{\upsilon}_{i,t,d}}},
\end{align}
where the first inequality follows from Lemma~\ref{pi}.

Further, from the definition of $e_{i,t}$, we have
\begin{equation}\label{133}
 x_{i,t+1}=\sum_{j=1}^n[W]_{ij}x_{j,t}+e_{i,t}.
\end{equation}
Now, from \eqref{W} and \eqref{133}, we have
\begin{align*}
  \bar{x}_{t+1}= \frac{1}{n}\sum_{i=1}^nx_{i,t+1} &= \frac{1}{n}\sum_{i=1}^n \sum_{j=1}^n[W]_{ij}x_{j,t}+\frac{1}{n}\sum_{i=1}^{n}e_{i,t} \\
  &= \frac{1}{n}\sum_{j=1}^n (\sum_{i=1}^n[W]_{ij})x_{j,t}+\frac{1}{n}\sum_{i=1}^{n}e_{i,t}\\
    &= \bar{x}_t+\bar{e}_t,\\
\end{align*}
where $\bar{e}_t= \frac{1}{n}\sum_{i=1}^n e_{i,t}$. Hence,
\begin{equation}\label{0}
  \bar{x}_{t+1}= \sum_{s=0}^{t}\bar{e}_s.
\end{equation}

It follows from \eqref{133} that
\begin{equation}\label{505}
  x_{i,t+1}=\sum_{s=0}^t\sum_{j=1}^n [W^{t-s}]_{ij}e_{i,s}.
\end{equation}

Now, using \eqref{505} and \eqref{0}, we have
\begin{equation}\label{dif}
  x_{i,t+1}-\bar{x}_{t+1}=\sum_{s=0}^t\sum_{j=1}^n([W^{t-s}]_{ij}-\frac{1}{n})e_{i,s}.
\end{equation}
Now, taking the Euclidean norm of \eqref{dif} and summing over $t\in \{1,\ldots,T\}$, one has:
\begin{align}\label{ppp}
\nonumber\sum_{t=1}^{T}\frac{1}{\alpha_t} \|\widehat{V}_{i,t}^{\frac{1}{4}} (x_{i,t+1}-\bar{x}_{t+1}) \|^2
&\stackrel{(i)}\leq \sum_{t=1}^{T} t \sum_{s=0}^{t}(\sum_{j=1}^n|[W^{t-s}]_{ij}-\frac{1}{n}|)^2 \frac{\|\widehat{V}_{i,s}^{\frac{1}{4}}e_{i,s}\|^2}{\alpha_s} \\ \nonumber
&\stackrel{(ii)}\leq \sum_{s=0}^{T}\frac{\|\widehat{V}_{i,s}^{\frac{1}{4}}e_{i,s}\|^2}{\alpha_s} \sum_{t=1}^{T}  nt\sigma^{2t-2s}_2(W)
\\\nonumber
&\stackrel{(iii)}\leq \frac{n}{(1-\sigma_2(W))^2}\sum_{t=0}^{T}\frac{\|\widehat{V}_{i,t}^{\frac{1}{4}}e_{i,t}\|^2}{\alpha_t}
\\\nonumber
&\stackrel{(iv)}\leq \frac{n}{(1-\sigma_2(W))^2}\sum_{d=1}^{p} \sum_{t=1}^{T}   \frac{\alpha_t m^2_{i,t,d}}{\sqrt{\widehat{\upsilon}_{i,t,d}}}
\\
&\stackrel{(v)}\leq \frac{n\alpha\sqrt{1+\log T}\sum_{d=1}^{p} \|g_{i,1:T,d} \|}{(1-\sigma_2(W))^2(1-\beta_1)(1-\eta)\sqrt{(1-\beta_2)(1-\beta_3)}},
\end{align}
where step (i) follows from $\| \sum_{i=1}^{n}a_i\|^2\leq n\sum_{i=1}^{n}\|a_i\|^2$, step (ii) follows from the following property of mixing matrix $W$ \cite{horn1990matrix},
\begin{equation}\label{www}
 \sum_{j=1}^n\left|\left[W^{t}\right]_{ij}-\frac{1}{n}\right| \leq \sqrt{n}\sigma^{t}_2(W),
\end{equation}
step (iii) follows from $\sum_{t=1}^{T} t\sigma^t_2(W)< \frac{1}{(1-\sigma_2(W))^2}$, step (iv) follows from \eqref{208} and step (v) follows from Lemma~\ref{000}.

Now, summing \eqref{ppp} over $i\in \mathcal{V}$ and using
\begin{equation}\label{nnc}
  \sum_{i=1}^{n} \|g_{i,1:T,d} \|\leq \sqrt{n}(\sum_{i=1}^{n} \|g_{i,1:T,d} \|^2)^{\frac{1}{2}}=\sqrt{n}\|g_{1:T,d} \|,
\end{equation}
we complete the proof.
\end{proof}
\begin{lem}\label{auxlemma}
For the sequence $x_{i,t}$ generated by Algorithm~\ref{alg} and the parameter settings and conditions assumed in Theorem~\ref{regret}, we have
 \begin{align} \label{eq:split}
 \nonumber
  \lefteqn{ \frac{1}{n}\sum_{i=1}^n\sum_{t=1}^T(\frac{\sqrt{\widehat{\upsilon}_{i,t}}}{\alpha_t(1-\beta_{1,t})}\|x^*_{t}-\sum_{j=1}^n[W]_{ij}x_{j,t}\|^2-\frac{\sqrt{\widehat{\upsilon}_{i,t}}}{\alpha_t(1-\beta_{1,t})}\|x^*_{t}-x_{i,t+1}\|^2)}\\
&\leq \frac{2\gamma_{\infty}^2}{\sqrt{n}(1-\beta_1)}\sum_{d=1}^{p} \frac{\sqrt{\widehat{\upsilon}_{T,d}}}{\alpha_{T}}+\frac{2\gamma_{\infty}}{\sqrt{n}}\sum_{d=1}^{p}\frac{ \sqrt{\widehat{\upsilon}_{T,d}}}{\alpha_{T}(1-\beta_{1})}\sum_{t=1}^{T-1}|x^*_{t+1,d}-x^*_{t,d}|+\frac{\gamma_{\infty}^2G_{\infty}}{\alpha(1-\lambda)^2(1-\beta_{1})^2}.
 \end{align}
\end{lem}
\begin{proof}
From the left side of \eqref{eq:split}, we have
\begin{align}
\nonumber &\frac{\sqrt{\widehat{\upsilon}_{i,t}}}{\alpha_t(1-\beta_{1,t})}\|x^*_{t}-\sum_{j=1}^n[W]_{ij}x_{j,t}\|^2-\frac{\sqrt{\widehat{\upsilon}_{i,t}}}{\alpha_t(1-\beta_{1,t})}\|x^*_{t}-x_{i,t+1}\|^2\\
&\nonumber=\frac{\sqrt{\widehat{\upsilon}_{i,t}}}{\alpha_{t}(1-\beta_{1,t})}\|x^*_{t}-\sum_{j=1}^n[W]_{ij}x_{j,t}\|^2-\frac{\sqrt{\widehat{\upsilon}_{i,t+1}}}{\alpha_{t+1}(1-\beta_{1,t+1})}\|x^*_{t+1}-\sum_{j=1}^n[W]_{ij}x_{j,t+1}\|^2 \\ \label{eq:seco:term}
&+\frac{\sqrt{\widehat{\upsilon}_{i,t+1}}}{\alpha_{t+1}(1-\beta_{1,t+1})}\|x^*_{t+1}-\sum_{j=1}^n[W]_{ij}x_{j,t+1}\|^2-\frac{\sqrt{\widehat{\upsilon}_{i,t+1}}}{\alpha_{t+1}(1-\beta_{1,t+1})}\|x^*_{t}-\sum_{j=1}^n[W]_{ij}x_{j,t+1}\|^2\\
 &+\underbrace{\frac{\sqrt{\widehat{\upsilon}_{i,t+1}}}{\alpha_{t+1}(1-\beta_{1,t+1})}\|x^*_{t}-x_{i,t+1}\|^2-\frac{\sqrt{\widehat{\upsilon}_{i,t}}}{\alpha_{t}(1-\beta_{1,t})}\|x^*_{t}-x_{i,t+1}\|^2}_{I_1}.\label{106}
\end{align}

By construction of \eqref{eq:seco:term}, we have
\begin{align}\label{eq11}
\nonumber&\sqrt{\widehat{\upsilon}_{i,t+1}}\|x^*_{t+1}-\sum_{j=1}^n[W]_{ij}x_{j,t+1}\|^2 - \sqrt{\widehat{\upsilon}_{i,t+1}}\|x^*_{t}-\sum_{j=1}^n[W]_{ij}x_{j,t+1}\|^2 \\
\nonumber&=\sum_{d=1}^{p}\sqrt{\widehat{\upsilon}_{i,t+1,d}}\langle x^*_{t+1,d}-x^*_{t,d},x^*_{t+1,d}+x^*_{t,d}-2\sum_{j=1}^n[W]_{ij}x_{j,t+1,d}\rangle\\
&\leq2\gamma_{\infty}\sum_{d=1}^{p}\sqrt{\widehat{\upsilon}_{i,t+1,d}}|x^*_{t+1,d}-x^*_{t,d}|,
\end{align}
where the last inequality holds due to \eqref{eq:diam}. Now we look at term~$I_1$:
\begin{align}\label{ggb}
 \nonumber  \sum_{t=1}^{T-1} I_1 &=\sum_{t=1}^{T-1}[\frac{\sqrt{\widehat{\upsilon}_{i,t+1}}}{\alpha_{t+1}(1-\beta_{1,t})}\|x^*_{t}-x_{i,t+1}\|^2-\frac{\sqrt{\widehat{\upsilon}_{i,t+1}}}{\alpha_{t+1}(1-\beta_{1,t})}\|x^*_{t}-x_{i,t+1}\|^2
 \\\nonumber&+\frac{\sqrt{\widehat{\upsilon}_{i,t+1}}}{\alpha_{t+1}(1-\beta_{1,t+1})}\|x^*_{t}-x_{i,t+1}\|^2-\frac{\sqrt{\widehat{\upsilon}_{i,t}}}{\alpha_{t}(1-\beta_{1,t})}\|x^*_{t}-x_{i,t+1}\|^2] \\\nonumber&\stackrel{(i)}\leq  \frac{1}{(1-\beta_1)}\sum_{t=1}^{T-1}[\frac{\sqrt{\widehat{\upsilon}_{i,t+1}}}{\alpha_{t+1}}\|x^*_{t}-x_{i,t+1}\|^2-\frac{\sqrt{\widehat{\upsilon}_{i,t}}}{\alpha_{t}}\|x^*_{t}-x_{i,t+1}\|^2]
 \\\nonumber&+\sum_{t=1}^{T-1}[\frac{\sqrt{\widehat{\upsilon}_{i,t+1}}\beta_{1,t}}{\alpha_{t+1}(1-\beta_{1})^2}\|x^*_{t}-x_{i,t+1}\|^2]\\
  & \stackrel{(ii)}\leq  \frac{\gamma_{\infty}^2}{(1-\beta_1)}\sum_{t=1}^{T-1}\sum_{d=1}^{p}(\frac{\sqrt{\widehat{\upsilon}_{i,t+1,d}}}{\alpha_{t+1}}-\frac{\sqrt{\widehat{\upsilon}_{i,t,d}}}{\alpha_t})
 +\frac{\gamma_{\infty}^2G_{\infty}}{\alpha(1-\lambda)^2(1-\beta_{1})^2},
\end{align}
where (i) follows from $\beta_{1,t}=\beta_1\lambda^{t-1},\lambda\in(0,1)$, $\beta_{1,t}\leq \beta_1$, by definition of $\widehat{\upsilon}_{i,t}$, we have
\begin{equation}\label{eq:dece:step}
\frac{\sqrt{\widehat{\upsilon}_{i,t+1,d}}}{\alpha_{t+1}}\geq\frac{\sqrt{\widehat{\upsilon}_{i,t,d}}}{\alpha_t},
\end{equation}
and
$$\frac{\sqrt{\widehat{\upsilon}_{i,t+1}}}{\alpha_{t+1}(1-\beta_{1,t+1})}\|x^*_{t}-x_{i,t+1}\|^2-\frac{\sqrt{\widehat{\upsilon}_{i,t+1}}}{\alpha_{t+1}(1-\beta_{1,t})}\|x^*_{t}-x_{i,t+1}\|^2\leq \frac{\sqrt{\widehat{\upsilon}_{i,t+1}}\beta_{1,t}}{\alpha_{t+1}(1-\beta_{1})^2}\|x^*_{t}-x_{i,t+1}\|^2.$$
Inequality (ii) follows from \eqref{eq:diam}, bounded gradients, $\|\nabla f_{i,t}(x_t)\|_{\infty}\leq G_{\infty}$, and the fact that
\begin{align*}
 \nonumber \sum_{t=1}^{T-1}\frac{\beta_{1,t}}{\alpha_{t+1}(1-\beta_{1})^2} &\leq \sum_{t=1}^{T-1}\frac{\beta_1 \lambda^{t-1}t}{\alpha(1-\beta_{1})^2}  \leq \frac{ 1}{\alpha(1-\lambda)^2(1-\beta_{1})^2}.
\end{align*}

Summing \eqref{106} over $t\in \{1,\ldots, T\}$, the first term telescopes, while \eqref{eq:seco:term} and $I_1$ terms are handled with \eqref{eq11} and \eqref{ggb}, respectively. Hence,
\begin{align}\label{az}
\nonumber&\lefteqn{\sum_{t=1}^T(\frac{\sqrt{\widehat{\upsilon}_{i,t}}}{\alpha_t(1-\beta_{1,t})}\|x^*_{t}-\sum_{j=1}^n[W]_{ij}x_{j,t}\|^2-\frac{\sqrt{\widehat{\upsilon}_{i,t}}}{\alpha_t(1-\beta_{1,t})}\|x^*_{t}-x_{i,t+1}\|^2)}\\
\nonumber&\leq \sum_{d=1}^{p}\frac{\gamma_{\infty}^2\sqrt{\widehat{\upsilon}_{i,1,d}}}{\alpha_1(1-\beta_{1})}+\sum_{t=1}^{T-1}\sum_{d=1}^{p}\frac{2\gamma_{\infty}\sqrt{\widehat{\upsilon}_{i,t+1,d}}}{\alpha_{t+1}(1-\beta_{1,t+1})}|x^*_{t+1,d}-x^*_{t,d}|
\\\nonumber&+\frac{\gamma_{\infty}^2}{(1-\beta_1)}\sum_{t=1}^{T-1}\sum_{d=1}^{p}(\frac{\sqrt{\widehat{\upsilon}_{i,t+1,d}}}{\alpha_{t+1}}-\frac{\sqrt{\widehat{\upsilon}_{i,t,d}}}{\alpha_t})+\frac{\gamma_{\infty}^2G_{\infty}}{\alpha(1-\lambda)^2(1-\beta_{1})^2}\\
&\leq \sum_{d=1}^{p}\frac{2\gamma_{\infty}^2\sqrt{\widehat{\upsilon}_{i,T,d}}}{\alpha_{T}(1-\beta_1)}
+\sum_{t=1}^{T-1}\sum_{d=1}^{p}\frac{2\gamma_{\infty}\sqrt{\widehat{\upsilon}_{i,t+1,d}}}{\alpha_{t+1}(1-\beta_{1,t+1})}|x^*_{t+1,d}-x^*_{t,d}|
+\frac{\gamma_{\infty}^2G_{\infty}}{\alpha(1-\lambda)^2(1-\beta_{1})^2}.
\end{align}
Now, summing \eqref{az} over $i\in \mathcal{V}$ and using the inequality $\sum_{i=1}^{n} \sqrt{\widehat{\upsilon}_{i,t}}\leq \sqrt{n}\sqrt{\widehat{\upsilon}_{t}}$, the claim in \eqref{eq:split} follows.
\end{proof}

\begin{lem}\label{functiondeviation}
Suppose Assumption~\ref{2040} holds and the parameters $\beta_1,\beta_2 \in [0,1)$ satisfy $\eta =\frac{\beta_1}{\sqrt{\beta_2}}<1$. Let $\beta_{1,t}=\beta_1\lambda^{t-1}, \lambda\in(0,1)$ and $\|\nabla f_{i,t}(x) \|_{\infty}\leq G_{\infty}$ for all $t\in\{1,\ldots,T\}$. Then, using a step-size $\alpha_t=\frac{\alpha}{\sqrt{t}}$ for the sequence $x_{i,t}$ generated by Algorithm~\ref{alg}, we have
\begin{align*}
\frac{1}{n}\sum_{i=1}^n\sum_{t=1}^T\Big(f_{i,t}(x_{i,t})-f_{i,t}(x^*_t)\Big)&\leq \frac{\alpha\sqrt{1+\log T}}{2\sqrt{(1-\beta_2)(1-\beta_3)}}\sum_{d=1}^p\|g_{1:T,d}\|
+\sum_{d=1}^p\frac{G_{\infty}\gamma_{\infty}(1+{\gamma_{\infty}}/{(2\alpha)})}{(1-\beta_1)^2(1-\lambda)^2}\\\nonumber&+
 \frac{2\alpha\sqrt{1+\log T}\sum_{d=1}^{p}\|g_{1:T,d} \|}{(1-\sigma_2(W))\sqrt{(1-\beta_1)}\sqrt{(1-\eta)}\sqrt{(1-\beta_2)(1-\beta_3)}}\\\nonumber&+\frac{\gamma_{\infty}^2}{\sqrt{n}}\sum_{d=1}^{p}\frac{\sqrt{T\widehat{\upsilon}_{T,d}}}{(1-\beta_1)\alpha}
+\frac{\gamma_{\infty}}{\sqrt{n}}\sum_{d=1}^{p}\frac{\sqrt{T\widehat{\upsilon}_{T,d}}}{(1-\beta_1)\alpha}\sum_{t=1}^{T-1}|x^*_{t+1,d}-x^*_{t,d}|.
\end{align*}
\end{lem}
\begin{proof}
From convexity of $f_{i,t}(\cdot)$, we have
\begin{align}\label{19}
\nonumber &\frac{1}{n}\sum_{i=1}^{n}\sum_{t=1}^T f_{i,t}(x_{i,t})-f_{i,t}(x^*_t) \leq\frac{1}{n}\sum_{i=1}^{n} \sum_{t=1}^T \langle\nabla f_{i,t}(x_{i,t}), x_{i,t}-x^*_t\rangle
\\\nonumber &=\frac{1}{n}\sum_{i=1}^{n}\sum_{t=1}^T\underbrace{\langle\nabla f_{i,t}(x_{i,t}), x_{i,t+1}-x^*_{t}\rangle}_{I_1}+\underbrace{\frac{1}{n}\sum_{i=1}^{n}\sum_{t=1}^T\langle\nabla f_{i,t}(x_{i,t}), x_{i,t}-\sum_{j=1}^n[W]_{ij}x_{j,t}\rangle}_{I_2}\\
&+\underbrace{\frac{1}{n}\sum_{i=1}^{n}\sum_{t=1}^T\langle\nabla f_{i,t}(x_{i,t}), \sum_{j=1}^n[W]_{ij}x_{j,t}-x_{i,t+1}\rangle}_{I_3}.
\end{align}

Individual terms in \eqref{19} can be bounded in the following way. From the Young's inequality for products\footnote{An elementary case of Young's inequality is  $ab\leq \frac{a^2}{2}+\frac{b^2}{2}$.}, we have
\begin{align}\label{my20}
\nonumber I_3&=\frac{1}{n}\sum_{i=1}^{n}\sum_{t=1}^T \langle\nabla f_{i,t}(x_{i,t}), \sum_{j=1}^n[W]_{ij}x_{j,t}-x_{i,t+1}\rangle \\
&\leq \frac{1}{n}\sum_{i=1}^{n}\sum_{t=1}^T\frac{\sqrt{\widehat{\upsilon}_{i,t}}}{2\alpha_t(1-\beta_{1,t})}\|\sum_{j=1}^n[W]_{ij}x_{j,t}-x_{i,t+1}\|^2
+\frac{1}{n}\sum_{i=1}^{n}\sum_{t=1}^T\frac{(1-\beta_{1,t})\alpha_t}{2\sqrt{\widehat{\upsilon}_{i,t}}}\|\nabla f_{i,t}(x_{i,t})\|^2.
\end{align}
Note also that:
\begin{align*}
 \sum_{t=1}^{T}\frac{\alpha_{t}g^2_{i,t,d}}{\sqrt{\widehat{\upsilon}_{i,t,d}}} \nonumber &= \sum_{t=1}^{T-1}\frac{\alpha_{t}g^2_{i,t,d}}{\sqrt{\widehat{\upsilon}_{i,t,d}}}+\frac{\alpha_{T}g^2_{i,T,d}}{\sqrt{(1-\beta_3)\max\{ \widehat{\upsilon}_{i,T-1,d},{\upsilon}_{i,T,d}\}+\beta_3\widehat{\upsilon}_{i,T-1,d}}} \\\nonumber &\leq \sum_{t=1}^{T-1}\frac{\alpha_{t}g^2_{i,t,d}}{\sqrt{\widehat{\upsilon}_{i,t,d}}}+\frac{\alpha_{T}g^2_{i,T,d}}{\sqrt{(1-\beta_3){\upsilon}_{i,T,d}}}\\\nonumber
  &= \sum_{t=1}^{T-1}\frac{\alpha_{t}g^2_{i,t,d}}{\sqrt{\widehat{\upsilon}_{i,t,d}}}+\frac{\alpha g^2_{i,T,d}}{\sqrt{T(1-\beta_2)(1-\beta_3)}\sqrt{\sum_{l=1}^{T}\beta_2^{T-l}g^2_{i,l,d}}} \\\nonumber
  &\leq \sum_{t=1}^{T-1}\frac{\alpha_{t}g^2_{i,t,d}}{\sqrt{\widehat{\upsilon}_{i,t,d}}}+\frac{\alpha |g_{i,T,d}|}{\sqrt{T(1-\beta_2)(1-\beta_3)}}.
\end{align*}
Using Cauchy-Schwarz
inequality, \eqref{log} and \eqref{nnc}, we have
\begin{align}\label{ff}
\nonumber \sum_{i=1}^{n}\sum_{t=1}^{T}\frac{\alpha_{t}g^2_{i,t,d}}{\sqrt{\widehat{\upsilon}_{i,t,d}}}
 &\leq \frac{\alpha}{\sqrt{(1-\beta_2)(1-\beta_3)}}\sum_{i=1}^{n}\sum_{t=1}^{T}\frac{ |g_{i,t,d}|}{\sqrt{t}}\\ \nonumber&\leq \frac{\alpha}{\sqrt{(1-\beta_2)(1-\beta_3)}}\sum_{i=1}^{n} \|g_{i,1:T,d}\|\sqrt{\sum_{t=1}^{T}\frac{1}{t}}\\
 \nonumber &\leq \frac{\alpha\sqrt{1+\log T}}{\sqrt{(1-\beta_2)(1-\beta_3)}}\sum_{i=1}^{n}\|g_{i,1:T,d}\|
  \\
  &\leq \frac{\sqrt{n}\alpha\sqrt{1+\log T}}{\sqrt{(1-\beta_2)(1-\beta_3)}}\|g_{1:T,d}\|.
\end{align}

Using \eqref{my20} and \eqref{ff}, we have
\begin{equation}\label{20}
 I_3 \leq   \frac{1}{n}\sum_{i=1}^{n}\sum_{t=1}^T \frac{\sqrt{\widehat{\upsilon}_{i,t}}}{2\alpha_t(1-\beta_1)}\|\sum_{j=1}^n[W]_{ij}x_{j,t}-z_{i,t+1}\|^2
+\frac{(1-\beta_1)\sqrt{n}\alpha\sqrt{1+\log T}}{2\sqrt{(1-\beta_2)(1-\beta_3)}}\|g_{1:T,d}\|.
\end{equation}

In addition, we have
\begin{align}
\nonumber I_2 &=  \frac{1}{n}\sum_{i=1}^{n}\sum_{t=1}^T\langle g_{i,t}, x_{i,t}-\sum_{j=1}^n[W]_{ij}x_{j,t}\rangle \\\nonumber&=\frac{1}{n}\sum_{i=1}^{n}\sum_{t=1}^T \langle g_{i,t}, x_{i,t}-\bar{x}_{t} + \bar{x}_{t}-\sum_{j=1}^n[W]_{ij}x_{j,t}\rangle\\
\nonumber&=\frac{1}{n}\sum_{i=1}^{n}\sum_{t=1}^T \langle g_{i,t}, x_{i,t}-\bar{x}_{t}\rangle+\frac{1}{n}\sum_{i=1}^{n}\sum_{t=1}^T \sum_{j=1}^n[W]_{ij}\langle g_{i,t}, \bar{x}_{t}-x_{j,t}\rangle\\
\nonumber &=\frac{1}{n}\sum_{i=1}^{n}\sum_{t=1}^T \langle{\sqrt{\alpha_{t}}{\widehat{V}}^{\frac{-1}{4}}_{i,t}}g_{i,t},{\frac{{\widehat{V}}^{\frac{1}{4}}_{i,t}}{\sqrt{\alpha_{t}}}} (x_{i,t}-\bar{x}_{t})\rangle
\\\nonumber&+\frac{1}{n}\sum_{i=1}^{n}\sum_{t=1}^T \sum_{j=1}^n[W]_{ij}\langle{\sqrt{\alpha_{t}}{\widehat{V}}^{\frac{-1}{4}}_{i,t}}g_{i,t},\frac{{\widehat{V}}^{\frac{1}{4}}_{i,t}}{\sqrt{\alpha_{t}}} (\bar{x}_{t}-x_{j,t})\rangle\\
\nonumber &\leq \frac{1}{n}\sum_{i=1}^{n}\sqrt{\sum_{t=1}^{T}{\alpha_{t}{\widehat{V}_{i,t}}^{\frac{-1}{2}}g^2_{i,t}}}\sqrt{\sum_{t=1}^T  \frac{{\widehat{V}_{i,t}}^{\frac{1}{2}}}{\alpha_{t}}{(x_{i,t}-\bar{x}_{t})}^2}
\\ &+\frac{1}{n}\sum_{i=1}^{n}\sum_{j=1}^n[W]_{ij}\sqrt{\sum_{t=1}^{T}{\alpha_{t}{\widehat{V}_{i,t}}^{\frac{-1}{2}}g^2_{i,t}}} \sqrt{\sum_{t=1}^T\frac{{\widehat{V}_{i,t}}^{\frac{1}{2}}}{\alpha_{t}}{(x_{j,t}-\bar{x}_{t})}^2},\label{c}
\end{align}

where \eqref{c} follows from Cauchy-Schwarz inequality.

Now, using \eqref{c} we obtain
\begin{align}
\nonumber I_2 &\leq \frac{1}{n}\sqrt{\sum_{i=1}^{n}\sum_{t=1}^{T}\sum_{d=1}^{p}{\alpha_{t}{\widehat{\upsilon}_{i,t,d}}^{\frac{-1}{2}}g^2_{i,t,d}}}\sqrt{\sum_{i=1}^{n}\sum_{t=1}^T  \sum_{d=1}^{p}\frac{{\widehat{\upsilon}_{i,t,d}}^{\frac{1}{2}}}{\alpha_{t}}{(x_{i,t,d}-\bar{x}_{t,d})}^2}
\\ &+\frac{1}{n}\sum_{j=1}^n[W]_{ij}\sqrt{\sum_{i=1}^{n}\sum_{t=1}^{T}\sum_{d=1}^{p}{\alpha_{t}{\widehat{\upsilon}_{i,t,d}}^{\frac{-1}{2}}g^2_{i,t,d}}} \sqrt{\sum_{i=1}^{n}\sum_{t=1}^T\sum_{d=1}^{p}\frac{{\widehat{\upsilon}_{i,t,d}}^{\frac{1}{2}}}{\alpha_{t}}{(x_{j,t,d}-\bar{x}_{t,d})}^2}\label{zz}\\
 &\leq
 \frac{2\alpha\sqrt{1+\log T}\sum_{d=1}^{p}\|g_{1:T,d} \|}{(1-\sigma_2(W))\sqrt{(1-\beta_1)}\sqrt{(1-\eta)}\sqrt{(1-\beta_2)(1-\beta_3)}}\label{21},
\end{align}
where \eqref{zz} utilize Cauchy-Schwarz inequality and \eqref{21} follows from \eqref{W}, Lemma \ref{50} and \eqref{ff}.

To bound $I_1$, using the update rule of $\widehat{m}_{i,t}$ in Algorithm~\ref{alg}, we have
\begin{align*}
&  \langle\alpha_t\frac{m_{i,t}}{\sqrt{\widehat{\upsilon}_{i,t}}}, x_{i,t+1}-x^*_t\rangle = \langle  \alpha_t(\frac{\beta_{1,t}}{\sqrt{\widehat{\upsilon}_{i,t}}}m_{i,t-1}+\frac{(1-\beta_{1,t})}{\sqrt{\widehat{\upsilon}_{i,t}}}\nabla f_{i,t}(x_{i,t})), x_{i,t+1}-x^*_t\rangle\\
  &=\langle  \frac{\alpha_t\beta_{1,t}}{\sqrt{\widehat{\upsilon}_{i,t}}}m_{i,t-1}, x_{i,t+1}-x^*_t\rangle
   +\langle  \frac{\alpha_t(1-\beta_{1,t})}{\sqrt{\widehat{\upsilon}_{i,t}}}\nabla f_{i,t}(x_{i,t}), x_{i,t+1}-x^*_t\rangle.
\end{align*}
Now, by rearranging the above equality, we obtain:
\begin{align}\label{700}
\nonumber &\sum_{d=1}^p  \langle  \frac{(1-\beta_{1,t})}{\sqrt{\widehat{\upsilon}_{i,t,d}}}\nabla f_{i,t,d}(x_{i,t,d}), x_{i,t+1,d}-x^*_{t,d}\rangle\\
 \nonumber  &= \sum_{d=1}^p \langle\frac{m_{i,t,d}}{\sqrt{\widehat{\upsilon}_{i,t,d}}}, x_{i,t+1,d}-x^*_{t,d}\rangle
+\sum_{d=1}^p \frac{\beta_{1,t}}{\sqrt{\widehat{\upsilon}_{i,t,d}}}  \langle {m}_{i,t-1,d}, x^*_{t,d} -x_{i,t+1,d}\rangle\\
\nonumber    &\leq \langle\frac{m_{i,t}}{\sqrt{\widehat{\upsilon}_{i,t}}}, x_{i,t+1}-x^*_{t}\rangle
+|| x^*_{t} -x_{i,t+1}||_{\infty}\sum_{d=1}^p\frac{\beta_{1,t}{m}_{i,t-1,d}}{\sqrt{\widehat{\upsilon}_{i,t,d}}}\\
   \nonumber    &\leq  \underbrace{\frac{1}{2\alpha_t}||x^*_{t}-\sum_{j=1}^n[W]_{ij}x_{j,t}||^2-\frac{1}{2\alpha_t}||x^*_{t}-x_{i,t+1}||^2
-\frac{1}{2\alpha_t}||x_{i,t+1}-\sum_{j=1}^n[W]_{ij}x_{j,t}||^2 }_{(II_1)}\\
& + \underbrace{|| x^*_{t} -x_{i,t+1}||_{\infty}G_{\infty}\sum_{d=1}^p\frac{\beta_{1,t}}{\sqrt{\widehat{\upsilon}_{i,t,d}}}}_{(II_2)},
\end{align}
where ($II_1$) follows from Lemma \ref{k}. The term ($II_2$) is obtained by induction as follows: using the assumption, $\|g_{i,t}\|_{\infty}\leq G_{\infty}$; now, using the update rule of $m_{i,t}$ in Algorithm~\ref{alg}, we have
\begin{equation}\label{boundm}
  \|m_{i,t}\|_{\infty}\leq (\beta_1+(1-\beta_1))\max(\| g_{i,t}\|_{\infty},\| m_{i,t-1}\|_{\infty})=\max(\| g_{i,t}\|_{\infty},\| m_{i,t-1}\|_{\infty})\leq G_{\infty},
\end{equation}
 where $\|m_{i,t-1}\|_{\infty} \leq G_{\infty}$ by induction hypothesis.

Next, using \eqref{700} yields the inequality:
\begin{align}\label{22}
\nonumber  I_1&=\langle \nabla f_{i,t}(x_{i,t}), x_{i,t+1}-x^*_{t}\rangle \\\nonumber &\leq \frac{\sqrt{\widehat{\upsilon}_{i,t}}}{(1-\beta_{1,t})}
  [ \frac{1}{2\alpha_t}||x^*_{t}-\sum_{j=1}^n[W]_{ij}x_{j,t}||^2-\frac{1}{2\alpha_t}||x^*_{t}-x_{i,t+1}||^2\\\nonumber&
-\frac{1}{2\alpha_t}||x_{i,t+1}-\sum_{j=1}^n[W]_{ij}x_{j,t}||^2] +
|| x^*_{t} -x_{i,t+1}||_{\infty}G_{\infty}\sum_{d=1}^p\frac{\beta_{1,t}}{(1-\beta_{1,t})}  \\
 \nonumber  &\leq  \frac{\sqrt{\widehat{\upsilon}_{i,t}}}{(1-\beta_{1,t})}
  [ \frac{1}{2\alpha_t}\|x^*_{t}-\sum_{j=1}^n[W]_{ij}x_{j,t}\|^2-\frac{1}{2\alpha_t}\|x^*_{t}-x_{i,t+1}\|^2]
\\&-\frac{\sqrt{\widehat{\upsilon}_{i,t}}}{(1-\beta_{1,t})2\alpha_t}||x_{i,t+1}-\sum_{j=1}^n[W]_{ij}x_{j,t}||^2
   +
\gamma_{\infty}G_{\infty}\sum_{d=1}^p\frac{\beta_{1,t}}{(1-\beta_{1,t})},
\end{align}
where the last line comes from~\eqref{eq:diam}.

Plugging \eqref{20}, \eqref{21}, and \eqref{22} into \eqref{19}, we obtain
\begin{align*}\label{pari}
\nonumber \frac{1}{n}\sum_{i=1}^{n}\sum_{t=1}^{T}\Big(f_{i,t}(x_{i,t})-f_{i,t}(x^*_t)\Big) &\leq \frac{\alpha\sqrt{1+\log T}}{2\sqrt{n}\sqrt{(1-\beta_2)(1-\beta_3)}}\sum_{d=1}^p\|g_{1:T,d}\|
+\gamma_{\infty}G_{\infty}\sum_{t=1}^{T}\sum_{d=1}^p\frac{\beta_{1,t}}{(1-\beta_{1,t})}\\\nonumber&+
 \frac{2\alpha\sqrt{1+\log T}\sum_{d=1}^{p}\|g_{1:T,d} \|}{(1-\sigma_2(W))\sqrt{(1-\beta_1)}\sqrt{(1-\eta)}\sqrt{(1-\beta_2)(1-\beta_3)}}\\&+
\frac{1}{n}\sum_{i=1}^{n}\sum_{t=1}^{T}\frac{\sqrt{\widehat{\upsilon}_{i,t}}}{2(1-\beta_{1,t})}[\frac{1}{\alpha_t}||x^*_{t}-\sum_{j=1}^n[W]_{ij}x_{j,t}||^2-\frac{1}{\alpha_t}||x^*_{t}-x_{i,t+1}||^2].
\end{align*}

Now, since $\beta_{1,t}=\beta_1\lambda^{t-1},\lambda\in(0,1)$ and $\beta_{1,t}\leq \beta_1$, we have
\begin{equation}\label{S}
  \sum_{t=1}^T\frac{\beta_{1,t}}{(1-\beta_{1,t})} \leq \sum_{t=1}^T \frac{\lambda^{t-1}}{(1-\beta_1)}\leq \frac{1}{(1-\beta_1)(1-\lambda)}.
\end{equation}

Finally, using Lemma \ref{auxlemma} and \eqref{S}, we obtain the desired result.
\end{proof}
\subsubsection{Proof of Theorem \ref{regret}}\label{B}
\begin{proof}
From convexity of $f_{i,t}(\cdot)$ it follows that
\begin{align*}
f_t(x_{i,t}) - f_t(x^*_t)&=f_t(x_{i,t})-f_t(\bar{x}_t)+f_t(\bar{x}_t) - f_t(x^*_t)
\\&\leq \langle g_{i,t},x_{i,t}-\bar{x}_t\rangle+f_t(\bar{x}_t)-f_t(x^*_t)\\
&= \frac{1}{n}\sum_{i=1}^nf_{i,t}(\bar{x}_t)-\frac{1}{n}\sum_{i=1}^nf_{i,t}(x^*_t)+\langle g_{i,t},x_{i,t}-\bar{x}_t\rangle,
\end{align*}
and
\begin{align}\label{666}
f_t(x_{i,t}) - f_t(x^*_t) &\leq \frac{1}{n}\sum_{i=1}^nf_{i,t}(x_{i,t})-\frac{1}{n}\sum_{i=1}^nf_{i,t}(x^*_t)+\langle g_{i,t},x_{i,t}-\bar{x}_t\rangle+\frac{1}{n}\sum_{i=1}^n\langle g_{i,t},x_{i,t}-\bar{x}_t\rangle.
\end{align}
Summing over $t\in \{1,\ldots,T\}$ and $i\in \mathcal{V}$, and applying Lemma \ref{functiondeviation} and \eqref{21} gives the desired result.
\end{proof}
\subsubsection{Proof of Theorem \ref{theorem2}}\label{B2}
\begin{proof}
We just need to rework the proof of Lemma~\ref{functiondeviation} and Theorem~\ref{regret} using stochastic gradients by tracking the changes. Indeed, using stochastic gradient at the beginning of Lemma \ref{functiondeviation}, we have
\begin{align*}
\nonumber &\sum_{t=1}^{T} \Big(f_{i,t}(x_{i,t})-f_{i,t}(x^*_{t})\Big)\leq \sum_{t=1}^{T}\langle\nabla f_{i,t}(x_{i,t}), x_{i,t}-x^*_{t}\rangle\\
\nonumber &=\sum_{t=1}^{T}\langle{\boldsymbol\nabla} f_{i,t}(x_{i,t}), x_{i,t}-x^*_{t}\rangle+\sum_{t=1}^{T}\langle\nabla f_{i,t}(x_{i,t})-{\boldsymbol\nabla} f_{i,t}(x_{i,t}), x_{i,t}-x^*_{t}\rangle\\
\nonumber &=\sum_{t=1}^{T}\langle{\boldsymbol\nabla} f_{i,t}(x_{i,t}), x_{i,t+1}-x^*_{t}\rangle+\sum_{t=1}^{T}\langle{\boldsymbol\nabla} f_{i,t}(x_{i,t}), x_{i,t}-\sum_{j=1}^n[W]_{ij}x_{j,t}\rangle\\
\nonumber&+\sum_{t=1}^{T}\langle{\boldsymbol\nabla} f_{i,t}(x_{i,t}), \sum_{j=1}^n[W]_{ij}x_{j,t}-x_{i,t+1}\rangle
 +\sum_{t=1}^{T}\langle\nabla f_{i,t}(x_{i,t})-{\boldsymbol\nabla} f_{i,t}(x_{i,t}), x_{i,t}-x^*_{t}\rangle.
\end{align*}
Further, if we replace any bound involving $G_{\infty}$ which is an upper bound on the exact gradient with the norm of stochastic gradient, we obtain
\begin{align}\label{exp}
 \frac{1}{n}\sum_{i=1}^{n}\sum_{t=1}^T\Big(f_{i,t}(x_{i,t})\nonumber-f_{i,t}(x^*_t)\Big) &\leq
  \frac{\alpha\sqrt{1+\log T}}{2\sqrt{n}\sqrt{(1-\beta_2)}}\sum_{d=1}^p\|{\boldsymbol g}_{1:T,d}\|
+\gamma_{\infty}\|m_{i,t-1}\|_{\infty}\sum_{t=1}^{T}\sum_{d=1}^p\frac{\beta_{1,t}}{(1-\beta_{1,t})}\\\nonumber&+
 \frac{2\alpha\sqrt{1+\log T}\sum_{d=1}^{p}\|{\boldsymbol g}_{1:T,d} \|}{(1-\sigma_2(W))\sqrt{(1-\beta_1)}\sqrt{(1-\eta)}\sqrt{(1-\beta_2)(1-\beta_3)}}\\\nonumber&+
\frac{1}{n}\sum_{i=1}^{n}\sum_{t=1}^{T}\frac{\sqrt{\widehat{\upsilon}_{i,t}}}{2(1-\beta_1)}[\frac{1}{\alpha_t}||x^*_{t}-\sum_{j=1}^n[W]_{ij}x_{j,t}||^2-\frac{1}{\alpha_t}||x^*_{t}-x_{i,t+1}||^2]
\\&+\frac{1}{n}\sum_{i=1}^{n}\langle\nabla f_{i,t}(x_{i,t})-{\boldsymbol \nabla}f_{i,t}(x_{i,t}), x_{i,t}-x^*_{t}\rangle.
\end{align}
Note that using Assumption~\ref{2030}, we have
$$\e{\langle\nabla f_{i,t}(x_{i,t})-{\boldsymbol \nabla}f_{i,t}(x_{i,t}), x_{i,t}-x^*_{t}\rangle}=0,$$
and
$$\e{\|{\boldsymbol \nabla}f_{i,t}(x_{i,t}) \|}\leq\Big(\e{\|{\boldsymbol \nabla}f_{i,t}(x_{i,t})\|^2}\Big)^{\frac{1}{2}}\leq \xi.$$
Hence taking expectation from \eqref{exp}, using the above inequality and \eqref{boundm}, we have
\begin{align*}
\frac{1}{n}\sum_{i=1}^{n} \sum_{t=1}^T&\Big(\e{f_{i,t}(x_{i,t})}-f_{i,t}(x^*_t)\Big) \leq   \frac{\alpha\sqrt{1+\log T}}{2\sqrt{n}\sqrt{(1-\beta_2)}}\sum_{d=1}^p\e{\|{\boldsymbol g}_{1:T,d}\|}
\\\nonumber&+\gamma_{\infty}\xi \sum_{t=1}^{T}\sum_{d=1}^p\frac{\beta_{1,t}}{(1-\beta_{1,t})}+
 \frac{2\alpha\sqrt{1+\log T}\sum_{d=1}^{p}\e{\|{\boldsymbol g}_{1:T,d} \|}}{(1-\sigma_2(W))\sqrt{(1-\beta_1)}\sqrt{(1-\eta)}\sqrt{(1-\beta_2)(1-\beta_3)}}\\ &+\frac{1}{n}\sum_{i=1}^{n}\frac{1}{2(1-\beta_1)}\e{\frac{\sqrt{\widehat{\upsilon}_{i,t}}}{\alpha_t}\|x^*_{t}-\sum_{j=1}^n[W]_{ij}x_{j,t}\|^2-\frac{\sqrt{\widehat{\upsilon}_{i,t}}}{\alpha_t}\|x^*_{t}-x_{i,t+1}\|^2}.
\end{align*}
According to the result in Lemma~\ref{auxlemma} and \eqref{S}, we have
\begin{align*}
&\frac{1}{n}\sum_{i=1}^n\sum_{t=1}^T\Big(\e{f_{i,t}(x_{i,t})}-f_{i,t}(x^*_t)\Big)\leq \frac{\alpha\sqrt{1+\log T}}{2\sqrt{n}\sqrt{(1-\beta_2)}}\sum_{d=1}^p\e{\|{\boldsymbol g}_{1:T,d}\|}
\\\nonumber&+\sum_{d=1}^p\frac{\xi\gamma_{\infty}}{(1-\beta_1)(1-\lambda)}+
 \frac{2\alpha\sqrt{1+\log T}\sum_{d=1}^{p}\e{\|{\boldsymbol g}_{1:T,d} \|}}{(1-\sigma_2(W))\sqrt{(1-\beta_1)}\sqrt{(1-\eta)}\sqrt{(1-\beta_2)(1-\beta_3)}}\\&+\frac{\gamma_{\infty}^2}{\sqrt{n}}\sum_{d=1}^{p}\frac{\sqrt{T}\e{\sqrt{\widehat{\upsilon}_{T,d}}}}{(1-\beta_1)\alpha}
+\frac{\gamma_{\infty}}{\sqrt{n}}\sum_{d=1}^{p}\frac{\sqrt{T}\e{\sqrt{\widehat{\upsilon}_{T,d}}} }{(1-\beta_1)\alpha}\sum_{t=1}^{T-1}|x^*_{t+1,d}-x^*_{t,d}|+\frac{\gamma_{\infty}^2\xi}{2\alpha(1-\lambda)^2(1-\beta_1)^2}.
\end{align*}
Finally, using \eqref{666} we complete the proof.
\end{proof}

Next, we establish a series of lemmas used in the proof of Theorems \ref{C} and \ref{theorem3}.

\begin{lem}
\cite{attouch2013convergence} Let $G_{\mathcal{X}}$ be the projected gradient defied in \eqref{333}. Then, $G_{\mathcal{X}}(\varpi,f_i,\alpha)=0$ if and only if $\varpi$ is a critical point of \eqref{125}.
\end{lem}

Next, we show that the projection map $G_\mathcal{X}$  in Definition \ref{def:G} is Lipschitz continuous.

\begin{lem}\label{888}
Suppose the second moment $\widehat{\upsilon}$ satisfies \eqref{v}. Let $x^+_{1,i}$ and $x^+_{2,i}$ be given in Definition \ref{def:G} with $\frac{m_{i}}{\sqrt{\widehat{\upsilon}_{i}}}$ replaced by $\frac{m_{i}^1}{\sqrt{\widehat{\upsilon}_{i}^1}}$ and $\frac{m_{i}^2}{\sqrt{\widehat{\upsilon}_{i}^2}}$, respectively. Then,
\begin{equation*}\label{55}
    \|G^1_{\mathcal{X}}(x_{i},\bar{f}_{i},\alpha)-G^{2}_{\mathcal{X}}(x_{i},\bar{f}_{i},\alpha)\|\leq \frac{\bar{\upsilon}}{\underline{\upsilon}} \| m_{i}^1-m_{i}^2\|,\qquad \forall i\in\mathcal{V},
\end{equation*}
where $G^1_{\mathcal{X}}$ and $G^{2}_{\mathcal{X}}$ are the projection maps corresponding to $x^+_{1,i}$ and $x^+_{2,i}$, respectively.
\end{lem}

\begin{proof}
Consider the optimality condition of \eqref{531}, for any $u\in \mathcal{X}$. For each $i\in\mathcal{V}$, observe that
\begin{equation}\label{13}
    \langle \frac{m_{i}^1}{\sqrt{\widehat{\upsilon}_{i}^1}}+\frac{1}{\alpha}(x^+_{1,i}-\sum_{j=1}^n[W]_{ij}x_{j}),u-x^+_{1,i}\rangle\geq 0,
\end{equation}
\begin{equation}\label{14}
    \langle \frac{m_{i}^2}{\sqrt{\widehat{\upsilon}_{i}^2}}+\frac{1}{\alpha}(x^+_{2,i}-\sum_{j=1}^n[W]_{ij}x_{j}),u-x^+_{2,i}\rangle\geq 0.
\end{equation}
Taking $u=x^+_{2,i}$ in \eqref{13}, we have
\begin{equation}\label{15}
    \langle m_{i}^1, x^+_{2,i}-x^+_{1,i}\rangle\geq \frac{\sqrt{\widehat{\upsilon}_{i}^1}}{\alpha}\langle \sum_{j=1}^n[W]_{ij}x_{j}-x^+_{1,i},x^+_{2,i}-x^+_{1,i} \rangle.
\end{equation}
Likewise, setting $u=x^+_{1,i}$ in \eqref{14}, we get
\begin{equation}\label{16}
    \langle m_{i}^2, x^+_{1,i}-x^+_{2,i}\rangle\geq \frac{\sqrt{\widehat{\upsilon}_{i}^2}}{\alpha}\langle \sum_{j=1}^n[W]_{ij}x_{j}-x^+_{2,i},x^+_{1,i}-x^+_{2,i} \rangle.
\end{equation}
Now, using \eqref{15}, \eqref{16} and \eqref{v}, we obtain
\begin{equation}\label{11}
    \| m_{i}^1-m_{i}^2\|\geq\frac{\underline{\upsilon}}{\alpha}\|x^+_{2,i}-x^+_{1,i}\|.
\end{equation}
Without loss of generality, assumes that $\sqrt{\widehat{\upsilon}_{i}^1}\leq \sqrt{\widehat{\upsilon}_{i}^2}$. Then, using \eqref{333}, we have
\begin{align*}
  \|G^1_{\mathcal{X}}(x_{i},\bar{f}_{i},\alpha)-G^{2}_{\mathcal{X}}(x_{i},\bar{f}_{i},\alpha)\| &=  \|\frac{\sqrt{\widehat{\upsilon}_{i}^1}}{\alpha}(x-x_{1,i}^+)-\frac{\sqrt{\widehat{\upsilon}_{i}^2}}{\alpha}(x-x_{2,i}^+) \|\\
   & \stackrel{(i)}\leq \frac{\bar{\upsilon}}{\alpha}\|x_{2,i}^+-x_{1,i}^+ \|\stackrel{(ii)}\leq
   \frac{\bar{\upsilon}}{\underline{\upsilon}} \| m_{i}^1-m_{i}^2\|,
\end{align*}
where (i) follows from \eqref{v} and (ii) follows from \eqref{11}.
\end{proof}

\begin{lem}\label{32}
Let $\beta_1,\beta_2 \in [0,1)$  satisfy $\eta = \frac{\beta_1}{\sqrt{\beta_2}}<1$. Then, for the sequence $x_{i,t}$ generated by Algorithm~\ref{alg}, we have
\begin{equation*}
  \frac{1}{n} \sum_{i=1}^{n} \|x_{i,t}-x^*_{t}\| \leq    \frac{2\sqrt{n}}{(1-\eta)\sqrt{(1-\beta_2)(1-\beta_3)}}\sum_{s=0}^{t-1}\alpha_s\sigma_2^{t-s-1}(W) := B_t.
\end{equation*}
\end{lem}
\begin{proof}
In the proof of Lemma~\ref{50} we showed that
\begin{equation*}\label{qqq}
  x_{i,t+1}-\bar{x}_{t+1}=\sum_{s=0}^t\sum_{j=1}^n([W^{t-s}]_{ij}-\frac{1}{n})e_{i,s}.
\end{equation*}
Now, using the above equality, we have
\begin{equation}\label{xc}
 \| x_{i,t+1}-\bar{x}_{t+1}\|=\sum_{s=0}^t\sum_{j=1}^n|[W^{t-s}]_{ij}-\frac{1}{n}|\|e_{i,s}\|.
\end{equation}
Also, using Lemma~\ref{000}, we get
\begin{equation}\label{ss}
  \frac{m_{i,t}}{\sqrt{\widehat{\upsilon}_{i,t}}}\leq \frac{1}{(1-\eta)\sqrt{(1-\beta_2)(1-\beta_3)}}.
\end{equation}
The above inequality together with the update rule of $x_{i,t+1}$, imply that
\begin{align}\label{e}
 \|e_{i,t}\|= \| x_{i,t+1}-\sum_{j=1}^n[W]_{ij}x_{j,t} \|\nonumber&=  \|\Pi_{\mathcal{X}}[\sum_{j=1}^n[W]_{ij}x_{j,t}-\alpha_t \frac{m_{i,t}}{\sqrt{\widehat{\upsilon}_{i,t}}}]-\sum_{j=1}^n[W]_{ij}x_{j,t} \| \\
 \nonumber  &\leq \|\sum_{j=1}^n[W]_{ij}x_{j,t}-\alpha_t \frac{m_{i,t}}{\sqrt{\widehat{\upsilon}_{i,t}}}-\sum_{j=1}^n[W]_{ij}x_{j,t} \| \\
 \nonumber  &= \alpha_t \|\frac{m_{i,t}}{\sqrt{\widehat{\upsilon}_{i,t}}} \| \\
   &\leq  \frac{\alpha_t}{(1-\eta)\sqrt{(1-\beta_2)(1-\beta_3)}},
\end{align}
where the first inequality follows from the nonexpansiveness property of the Euclidean projection\footnote{$\|\Pi_{\mathcal{X}}[x_1]-\Pi_{\mathcal{X}}[x_2]\|\leq \|x_1-x_2 \|,\,\,\,\forall x_1,x_2 \in \mathbb{R}^p.$}.

Substituting \eqref{www} and \eqref{e} into \eqref{xc}, we get
$$\|x_{i,t}- \bar{x}_{t}\|\leq  \frac{\sqrt{n}}{(1-\eta)\sqrt{(1-\beta_2)(1-\beta_3)}}\sum_{s=0}^{t-1}\alpha_s\sigma_2^{t-s-1}(W).$$
Summing the above inequality over $i\in \mathcal{V}$, we conclude that
\begin{equation*}
\sum_{i=1}^{n} \|x_{i,t}-x^*_{t}\|
                  \leq \sum_{i=1}^{n}\|x_{i,t}-\bar{x}_{t}\|+\sum_{i=1}^{n}\|\bar{x}_{t}-x^*_{t}\|
                \leq \frac{2n\sqrt{n}}{(1-\eta)\sqrt{(1-\beta_2)(1-\beta_3)}}\sum_{s=0}^{t-1}\alpha_s\sigma_2^{t-s-1}(W).
                  \end{equation*}
\end{proof}
\begin{lem}\label{parvin}
For the sequence
$x_{i,t}$ generated by Algorithm~\ref{alg}, we have
\begin{equation*}\label{114}
  \langle \nabla \bar{f}_{i,t},G_{\mathcal{X}}(x_{i,t},\bar{f}_{i,t},\alpha_t) \rangle  \geq  \frac{(2-\beta_{1,t})}{2(1-\beta_{1,t})}\|G_{\mathcal{X}}(x_{i,t},\bar{f}_{i,t},\alpha_t) \|^2
  -\frac{\beta_{1,t}\widehat{\upsilon}_{i,t}}{2(1-\beta_{1,t})(1-\eta)^2(1-\beta_2)}.
\end{equation*}
\end{lem}
\begin{proof}
A quick look at optimality condition of \eqref{531} verifies that
$$ \langle \frac{m_{i,t}}{\sqrt{\widehat{\upsilon}_{i,t}}}+\frac{1}{\alpha_t}(x_{i,t+1}-\sum_{j=1}^n[W]_{ij}x_{j,t}),z-x_{i,t+1}\rangle \geq 0,\qquad \forall z\in \mathcal{X}.$$
Substituting $z = x_{i,t}$ into the above inequality, we get
\begin{equation*}
    \langle \frac{ m_{i,t}}{\sqrt{\widehat{\upsilon}_{i,t}}},x_{i,t}-x_{i,t+1}\rangle
   \geq  \frac{1}{\alpha_t}\langle x_{i,t+1}-\sum_{j=1}^n[W]_{ij}x_{j,t},x_{i,t+1}-x_{i,t}\rangle.
\end{equation*}

Now using \eqref{W}, we have
\begin{align*}
& \frac{\sqrt{\widehat{\upsilon}_{i,t}}}{\alpha_t}\|x_{i,t}- x_{i,t+1}\|^2 \leq \langle  m_{i,t},x_{i,t}-x_{i,t+1} \rangle
   = \langle \beta_{1,t}m_{i,t-1}+(1-\beta_{1,t})\nabla \bar{f}_{i,t}, x_{i,t}-x_{i,t+1}\rangle\\
  \nonumber  &=   \frac{\beta_{1,t}\sqrt{\widehat{\upsilon}_{i,t-1}}}{\alpha_{t}}\langle \frac{\alpha_{t}{m}_{i,t-1}}{\sqrt{\widehat{\upsilon}_{i,t-1}}}, x_{i,t}-x_{i,t+1}\rangle
  + (1-\beta_{1,t})\langle \nabla \bar{f}_{i,t}, x_{i,t}-x_{i,t+1}\rangle \\
  \nonumber  &\leq   \frac{\beta_{1,t}\sqrt{\widehat{\upsilon}_{i,t}}}{2\alpha_{t}}(\frac{\alpha_t^2}{(1-\eta)^2(1-\beta_2)}+\|x_{i,t}-x_{i,t+1}\|^2)
  +(1-\beta_{1,t})\langle \nabla \bar{f}_{i,t}, x_{i,t}-x_{i,t+1}\rangle,
\end{align*}
where the first equality follows from the update rule of $m_{i,t}$. The last inequality is valid
since Cauchy-Schwarz inequality and inequality $\sqrt{\widehat{\upsilon}_{i,t-1}}\beta_{1,t}\leq \sqrt{\widehat{\upsilon}_{i,t}}\beta_1$. The
claim then follows after using Definition~\ref{def:G}.
\end{proof}
\subsubsection{Proof of Theorem \ref{C}}
\begin{proof}
Let $\bar{f}_{i,t}(x_{i,t})=\frac{1}{t}\sum_{s=1}^tf_{i,s}(x_{i,t})$. Using Assumption~\ref{2020}, Taylor’s expansion and Definition~\ref{def:G}, we have
\begin{align*}
    \bar{f}_{i,t}(x_{i,t+1})
    &\leq \bar{f}_{i,t}(x_{i,t})+\langle \nabla \bar{f}_{i,t}(x_{i,t}),x_{i,t+1}-x_{i,t}\rangle+\frac{\rho}{2}\| x_{i,t+1}-x_{i,t}\|^2\\&
    \leq \bar{f}_{i,t}(x_{i,t})-\frac{\alpha_t}{\sqrt{\widehat{\upsilon}_{i,t}}}\langle \nabla \bar{f}_{i,t}(x_{i,t}),G_{\mathcal{X}}(x_{i,t},\bar{f}_{i,t},\alpha_t)\rangle+\frac{\rho\alpha_t^2}{2\widehat{\upsilon}_{i,t}}\| G_{\mathcal{X}}(x_{i,t},\bar{f}_{i,t},\alpha_t)\|^2.
\end{align*}
The above inequality together with Lemma \ref{parvin}, \eqref{v} and $\beta_{1,t}\leq\beta_1$, imply that
\begin{align}\label{a!b!c35}
    \bar{f}_{i,t}(x_{i,t+1})&\leq\nonumber \bar{f}_{i,t}(x_{i,t})
     - \frac{\alpha_t(2-\beta_1)}{2\bar{\upsilon}}\|G_{\mathcal{X}}(x_{i,t},\bar{f}_{i,t},\alpha_t) \|^2\\&+\frac{\alpha_t\beta_{1,t}\bar{\upsilon}}{2(1-\beta_{1,t})(1-\eta)^2(1-\beta_2)}
    +\frac{\rho\alpha_t^2}{2\underline{\upsilon}^2}\| G_{\mathcal{X}}(x_{i,t},\bar{f}_{i,t},\alpha_t)\|^2.
\end{align}
Let $\Delta_{i,t}= \bar{f}_{i,t}(x_{i,t})-\bar{f}_{i,t}(x_t^*)$ denotes the instantaneous loss at round $t$. We have
\begin{equation}\label{306}
 \Delta_{i,t+1}=\frac{t}{t+1} \underbrace{(\bar{f}_{i,t}(x_{i,t+1})-\bar{f}_{i,t}(x^*_{t+1}))}_{(I)}+\frac{1}{t+1} (f_{i,t+1}(x_{i,t+1})-f_{i,t+1}(x^*_{t+1})).
\end{equation}

Note that ($I$) can be bounded as follows:
\begin{align}\label{3077}
\nonumber \bar{f}_{i,t}(x_{i,t+1})-\bar{f}_{i,t}(x^*_{t+1})\leq \bar{f}_{i,t}(x_{i,t+1})-\bar{f}_{i,t}(x^*_{t})&\leq  \Delta_{i,t}
-[ \frac{(2-\beta_1)\alpha_t}{2\bar{\upsilon}}-\frac{\rho \alpha_t^2}{2\underline{\upsilon}^2}]\|G_{\mathcal{X}}(x_{i,t},\bar{f}_{i,t},\alpha_t) \|^2
\\&+\frac{\alpha_t\beta_{1,t}\bar{\upsilon}}{2(1-\beta_{1,t})(1-\eta)^2(1-\beta_2)},
\end{align}
where the first inequality is due to the optimality of $ x_t^*$ and
the second inequality follows from \eqref{a!b!c35}.

Now, combining \eqref{3077} and \eqref{306} gives
\begin{align*}
  \Delta_{i,t+1} &\leq  \frac{t}{t+1}
  (\Delta_{i,t}-[ \frac{(2-\beta_1)\alpha_t}{2\bar{\upsilon}}-\frac{\rho \alpha_t^2}{2\underline{\upsilon}^2}]\|G_{\mathcal{X}}(x_{i,t},\bar{f}_{i,t},\alpha_t) \|^2
\\&+\frac{\alpha_t\beta_{1,t}\bar{\upsilon}}{2(1-\beta_{1,t})(1-\eta)^2(1-\beta_2)}
  )+\frac{1}{t+1}(f_{i,t+1}(x_{i,t+1})-f_{i,t+1}(x^*_{t+1})).
  \end{align*}
By rearranging the above inequality, we have
  \begin{align}\label{566}
   \frac{t}{t+1}[ \frac{(2-\beta_1)\alpha_t}{2\bar{\upsilon}}\nonumber -\frac{\rho \alpha_t^2}{2\underline{\upsilon}^2}]\|G_{\mathcal{X}}(x_{i,t},\bar{f}_{i,t},\alpha_t)\|^2
   &\leq \frac{t}{t+1}(\Delta_{i,t}+\frac{\alpha_t\beta_{1,t}\bar{\upsilon}}{2(1-\beta_{1,t})(1-\eta)^2(1-\beta_2)})\\&
 +\frac{1}{t+1}(f_{i,t+1}(x_{i,t+1})-f_{i,t+1}(x^*_{t+1}))-\Delta_{i,t+1}.
\end{align}
Using the definition of $\Delta_{i,t+1}$, we get $\frac{1}{t+1}(f_{i,t+1}(x_{i,t+1})-f_{i,t+1}(x^*_{t+1}))-\Delta_{i,t+1}= -(\frac{t}{t+1})(\bar{f}_{i,t}(x_{i,t+1})-\bar{f}_{i,t}(x^*_{t+1}))$. Hence, using \eqref{566} and simplifying terms, we obtain
\begin{align}\label{38}
  \nonumber [ \frac{(2-\beta_1)\alpha_t}{2\bar{\upsilon}}-\frac{\rho \alpha_t^2}{2\underline{\upsilon}^2}] \|&G_{\mathcal{X}}(x_{i,t},\bar{f}_{i,t},\alpha_t)\|^2 \\&\leq \Delta_{i,t} -(\bar{f}_{i,t}(x_{i,t+1})-\bar{f}_{i,t}(x^*_{t+1})) +\frac{\alpha_t\beta_{1,t}\bar{\upsilon}}{2(1-\beta_{1,t})(1-\eta)^2(1-\beta_2)}.
\end{align}
Note that:
\begin{align}
\nonumber  &\frac{1}{n}\sum_{i=1}^{n}\sum_{t=1}^T \Big(\Delta_{i,t} -(\bar{f}_{i,t}(x_{i,t+1})-\bar{f}_{i,t}(x^*_{t+1}))\Big)\\\nonumber&= \frac{1}{n}\sum_{i=1}^{n}\sum_{t=1}^T \Big((\bar{f}_{i,t}(x_{i,t})-\bar{f}_{i,t}(x_{i,t+1}))-(\bar{f}_{i,t}(x^*_t)-\bar{f}_{i,t}(x^*_{t+1}))\Big) \\
 \nonumber  &=\frac{1}{n}\sum_{i=1}^{n}[ -\bar{f}_{i,T}(x_{i,T+1}) +\bar{f}_{i,1}(x_{i,1})-\bar{f}_{i,1}(x^*_{1})+\bar{f}_{i,T}(x^*_{T+1})]\\
   &+\frac{1}{n}\sum_{i=1}^{n} \sum_{t=2}^T t^{-1}\Big(f_{i,t}(x_{i,t})-\bar{f}_{i,t-1}(x_{i,t})-(f_{i,t}(x^*_t)-\bar{f}_{i,t-1}(x^*_t))\Big)\label{asas} \\
   &\leq \frac{1}{n}\sum_{i=1}^{n} L\Big(\| x_{i,T+1}-x_{T+1}^*\|+\| x_{i,1}-x^*_{1}\|+\sum_{t=2}^T 2t^{-1}\| x_{i,t}-x^*_{t} \|\Big)\label{807}\\
   &\leq  2L\max_{t\in \{2,\ldots,T\}}B_t(1+\sum_{t=2}^Tt^{-1})\leq 2L\max_{t\in \{2,\ldots,T\}}B_t(2+\log T)\label{808},
\end{align}
where the first equality uses $$ \bar{f}_{i,t}(x_{i,t})-\bar{f}_{i,t-1}(x_{i,t})=t^{-1}(f_{i,t}(x_{i,t})-\bar{f}_{i,t-1}(x_{i,t})).$$
The first inequality follows from Assumption~\ref{2020}. The second inequality follows from Lemma \ref{32} and \eqref{log}.

Summing \eqref{38} over $i\in\mathcal{V}$, $t\in\{1,\ldots,T\}$ and using \eqref{808}, we have
\begin{align}\label{901}
\nonumber &\frac{1}{n}\sum_{i=1}^{n}\min_{t\in\{1,\ldots,T\}} \|G_{\mathcal{X}}(x_{i,t},\bar{f}_{i,t},\alpha_t)\|^2\sum_{t=1}^T [ \frac{(2-\beta_1)\alpha_t}{2\bar{\upsilon}}-\frac{\rho \alpha_t^2}{2\underline{\upsilon}^2}]
  \\ \nonumber &\leq\frac{1}{n}\sum_{i=1}^{n} \sum_{t=1}^T [ \frac{(2-\beta_1)\alpha_t}{2\bar{\upsilon}}-\frac{\rho \alpha_t^2}{2\underline{\upsilon}^2}] \|G_{\mathcal{X}}(x_{i,t},\bar{f}_{i,t},\alpha_t)\|^2
\\\nonumber&  \leq
(2+\log T)2L\max_{t\in \{2,\ldots,T\}} \frac{2\sqrt{n}}{(1-\eta)\sqrt{(1-\beta_2)(1-\beta_3)}}\sum_{s=0}^{t-1}\alpha_s\sigma_2^{t-s-1}(W) \\&+\sum_{t=1}^{T}\frac{\alpha_t\beta_{1,t}\bar{\upsilon}}{2(1-\beta_{1,t})(1-\eta)^2(1-\beta_2)}.
\end{align}

Note that $\sum_{t=1}^T [ \frac{(2-\beta_1)\alpha_t}{2\bar{\upsilon}}-\frac{\rho \alpha_t^2}{2\underline{\upsilon}^2}]>0$. Therefore, dividing both sides of the \eqref{901}
by $\sum_{t=1}^T [ \frac{(2-\beta_1)\alpha_t}{2\bar{\upsilon}}-\frac{\rho \alpha_t^2}{2\underline{\upsilon}^2}]$, we obtain \eqref{nooon}.
\end{proof}
\subsubsection{Proof of Corollary \ref{Corollary3}}\label{C3}
\begin{proof}
With the constant step-sizes
$\alpha_t=\frac{(2-\beta_1)\underline{\upsilon}^2}{2\rho\bar{\upsilon}}$ for all $t\in\{1,\ldots,T\}$, we have
$$
\vartheta_t= \sum_{t=1}^T [ \frac{(2-\beta_1)\alpha_t}{2\bar{\upsilon}}-\frac{\rho \alpha_t^2}{2\underline{\upsilon}^2}]=\frac{T(2-\beta_1)^2\underline{\upsilon}^2}{8\rho\bar{\upsilon}^2}.
 $$
Therefore, using the above equality and \eqref{nooon} together with $\alpha_t=\frac{(2-\beta_1)\underline{\upsilon}^2}{2\rho\bar{\upsilon}}$ for all $t\in\{1,\ldots,T\}$, we obtain
 \begin{align}\label{agg}
  \frac{1}{\vartheta_t} \sum_{t=1}^{T}\frac{\alpha_t\beta_{1,t}\bar{\upsilon}}{2(1-\beta_{1,t})(1-\eta)^2(1-\beta_2)} & \nonumber = \frac{4\bar{\upsilon}}{T(2-\beta_1)}\sum_{t=1}^{T}\frac{\beta_{1,t}\bar{\upsilon}}{2(1-\beta_{1,t})(1-\eta)^2(1-\beta_2)}\\
    & \nonumber \leq \frac{4\bar{\upsilon}}{T(2-\beta_1)}\sum_{t=1}^{T}\frac{\lambda^{t-1}\bar{\upsilon}}{2(1-\beta_{1})(1-\eta)^2(1-\beta_2)}\\
    &  \leq \frac{2\bar{\upsilon}^2}{T(2-\beta_1)(1-\beta_{1})(1-\eta)^2(1-\beta_2)(1-\lambda)},
 \end{align}
 and
 \begin{align}\label{aon}
&\nonumber\frac{1}{\vartheta_t}(2+\log T)2L\max_{t\in \{2,\ldots,T\}} \frac{2\sqrt{n}}{(1-\eta)\sqrt{(1-\beta_2)(1-\beta_3)}}\sum_{s=0}^{t-1}\alpha_s\sigma_2^{t-s-1}(W)\\
\nonumber &\leq\frac{16\bar{\upsilon}(2+\log T)L \sqrt{n}}{T(2-\beta_1)(1-\eta)\sqrt{(1-\beta_2)(1-\beta_3)}}\max_{t\in \{2,\ldots,T\}} \sum_{s=0}^{t-1}\sigma_2^{t-s-1}(W)\\&\leq
 \frac{16\bar{\upsilon}(2+\log T)L\sqrt{n} }{T(2-\beta_1)(1-\eta)\sqrt{(1-\beta_2)(1-\beta_3)}(1-\sigma_2(W))}.
\end{align}
Combining \eqref{agg} and \eqref{aon} now proves the desired result.
\end{proof}
\subsubsection{Proof of Theorem \ref{theorem3}}\label{BB3}
\begin{proof}
Let
 $$\delta_{i,t}= {\boldsymbol \nabla}\bar{f}_{i,t}(x_{i,t})- \nabla \bar{f}_{i,t}(x_{i,t}), \qquad \forall t\geq 1,$$
where $ {\boldsymbol \nabla}\bar{f}_{i,t}(x_{i,t})= \frac{1}{t}\sum_{s=1}^t {\boldsymbol \nabla}f_{i,s}(x_{i,t})$ denotes the mini-batch stochastic gradient on node $i$.

Since $f_{i,t}$ is $\rho$-smooth, it follows that $\bar{f}_{i,t}$ is also $\rho$-smooth. Hence for every $i\in \mathcal{V}$ and $t\in\{1,\ldots,T\}$, we have
\begin{align*}
    \bar{f}_{i,t}(x_{i,t+1})
    &\leq \bar{f}_{i,t}(x_{i,t})+\langle \nabla \bar{f}_{i,t}(x_{i,t}),x_{i,t+1}-x_{i,t}\rangle+\frac{\rho}{2}\| x_{i,t+1}-x_{i,t}\|^2
    \\&\leq \bar{f}_{i,t}(x_{i,t})-\frac{\alpha_t}{\sqrt{\widehat{\upsilon}_{i,t}}}\langle \nabla \bar{f}_{i,t}(x_{i,t}),\mathbf{G}_{\mathcal{X}}(x_{i,t},\bar{f}_{i,t},\alpha_t)\rangle+\frac{\rho\alpha_t^2}{2\widehat{\upsilon}_{i,t}}\| \mathbf{G}_{\mathcal{X}}(x_{i,t},\bar{f}_{i,t},\alpha_t)\|^2\\
    &= \bar{f}_{i,t}(x_{i,t})-\frac{\alpha_t}{\sqrt{\widehat{\upsilon}_{i,t}}}\langle \boldsymbol \nabla \bar{f}_{i,t}(x_{i,t}),\mathbf{G}_{\mathcal{X}}(x_{i,t},\bar{f}_{i,t},\alpha_t)\rangle+\frac{\rho\alpha_t^2}{2\widehat{\upsilon}_{i,t}}\| \mathbf{G}_{\mathcal{X}}(x_{i,t},\bar{f}_{i,t},\alpha_t)\|^2\\
    &+\frac{\alpha_t}{\sqrt{\widehat{\upsilon}_{i,t}}} \langle \delta_{i,t},\mathbf{G}_{\mathcal{X}}(x_{i,t},\bar{f}_{i,t},\alpha_t) \rangle,
\end{align*}
where $\mathbf{G}_{\mathcal{X}}(x_{i,t},\bar{f}_{i,t},\alpha_t)$ denotes the projected stochastic gradient on node $i$.

The above inequality together with Lemma \ref{parvin}, \eqref{v} and $\beta_{1,t}\leq\beta_1$, imply that
\begin{align}\label{35}
   \nonumber\bar{f}_{i,t}(x_{i,t+1})\nonumber&\leq\nonumber \bar{f}_{i,t}(x_{i,t})
     - \frac{\alpha_t(2-\beta_1)}{2\bar{\upsilon}}\|\mathbf{G}_{\mathcal{X}}(x_{i,t},\bar{f}_{i,t},\alpha_t) \|^2+\frac{\alpha_t\beta_{1,t}\bar{\upsilon}}{2(1-\beta_{1,t})(1-\eta)^2(1-\beta_2)}
   \\
\nonumber    & +\frac{\rho\alpha_t^2}{2\underline{\upsilon}^2}\| \mathbf{G}_{\mathcal{X}}(x_{i,t},\bar{f}_{i,t},\alpha_t)\|^2+\frac{\alpha_t}{\sqrt{\widehat{\upsilon}_{i,t}}} \langle \delta_{i,t},G_{\mathcal{X}}(x_{i,t},\bar{f}_{i,t},\alpha_t) \rangle
  \\
\nonumber    &  +\frac{\alpha_t}{\sqrt{\widehat{\upsilon}_{i,t}}} \langle \delta_{i,t},\mathbf{G}_{\mathcal{X}}(x_{i,t},\bar{f}_{i,t},\alpha_t)-G_{\mathcal{X}}(x_{i,t},\bar{f}_{i,t},\alpha_t) \rangle\\
    &\leq\nonumber \bar{f}_{i,t}(x_{i,t})
     - \frac{\alpha_t(2-\beta_1)}{2\bar{\upsilon}}\|\mathbf{G}_{\mathcal{X}}(x_{i,t},\bar{f}_{i,t},\alpha_t) \|^2
+\frac{\alpha_t\beta_{1,t}\bar{\upsilon}}{2(1-\beta_{1,t})(1-\eta)^2(1-\beta_2)}
   \\
\nonumber    & +\frac{\rho\alpha_t^2}{2\underline{\upsilon}^2}\| \mathbf{G}_{\mathcal{X}}(x_{i,t},\bar{f}_{i,t},\alpha_t)\|^2+\frac{\alpha_t}{\sqrt{\widehat{\upsilon}_{i,t}}} \langle \delta_{i,t},G_{\mathcal{X}}(x_{i,t},\bar{f}_{i,t},\alpha_t) \rangle
   \\
\nonumber    & +\frac{\alpha_t}{\sqrt{\widehat{\upsilon}_{i,t}}}\| \delta_{i,t}\| \|\mathbf{G}_{\mathcal{X}}(x_{i,t},\bar{f}_{i,t},\alpha_t)-G_{\mathcal{X}}(x_{i,t},\bar{f}_{i,t},\alpha_t) \|\\
    &\leq\nonumber \bar{f}_{i,t}(x_{i,t})
     - \frac{\alpha_t(2-\beta_1)}{2\bar{\upsilon}}\|\mathbf{G}_{\mathcal{X}}(x_{i,t},\bar{f}_{i,t},\alpha_t) \|^2+\frac{\alpha_t\beta_{1,t}\bar{\upsilon}}{2(1-\beta_{1,t})(1-\eta)^2(1-\beta_2)}
  \\
\nonumber    &  +\frac{\rho\alpha_t^2}{2\underline{\upsilon}^2}\| \mathbf{G}_{\mathcal{X}}(x_{i,t},\bar{f}_{i,t},\alpha_t)\|^2
   +\frac{\alpha_t}{\sqrt{\widehat{\upsilon}_{i,t}}} \langle \delta_{i,t},G_{\mathcal{X}}(x_{i,t},\bar{f}_{i,t},\alpha_t) \rangle
 \\ &   +\frac{\alpha_t}{\sqrt{\widehat{\upsilon}_{i,t}}} \| \delta_{i,t}\|\frac{\bar{\upsilon}}{\underline{\upsilon}}\sum_{r=1}^t(1-\beta_{1,r})\beta_1^{t-r}\|\delta_{i,r}\|,
\end{align}
where the second inequality follows from the Cauchy-Schwartz inequality. The last inequality
follows from the fact that $\beta_{1,t}=\beta_1\lambda^{t-1}, \lambda\in(0,1)$ and Lemma \ref{888} if we set $G^{1}_{\mathcal{X}}(x_{i},\bar{f}_{i},\alpha)=\mathbf{G}_{\mathcal{X}}(x_{i,t},\bar{f}_{i,t},\alpha_t)$,  and $G^{2}_{\mathcal{X}}(x_{i},\bar{f}_{i},\alpha)=G_{\mathcal{X}}(x_{i,t},\bar{f}_{i,t},\alpha_t)$.

Let $\Delta_{i,t}= \bar{f}_{i,t}(x_{i,t})-\bar{f}_{i,t}(x_t^*) $ denotes the instantaneous stochastic loss at time $t$. We have
\begin{equation}\label{366}
 \Delta_{i,t+1}=\frac{t}{t+1}\underbrace{(\bar{f}_{i,t}(x_{i,t+1})-\bar{f}_{i,t}(x^*_{t+1}))}_{(I)}+\frac{1}{t+1}(f_{i,t+1}(x_{i,t+1})-f_{i,t+1}(x^*_{t+1})).
\end{equation}
Observe that ($I$) can be bounded as follows:
\begin{align}\label{307}
\nonumber & \bar{f}_{i,t}(x_{i,t+1})-\bar{f}_{i,t}(x^*_{t+1})\leq \bar{f}_{i,t}(x_{i,t+1})-\bar{f}_{i,t}(x^*_{t})\nonumber \\&\leq \Delta_{i,t}
-\nonumber[ \frac{(2-\beta_1)\alpha_t}{2\bar{\upsilon}}-\frac{\rho \alpha_t^2}{2\underline{\upsilon}^2}]\|G_{\mathcal{X}}(x_{i,t},\bar{f}_{i,t},\alpha_t) \|^2+\frac{\alpha_t\beta_{1,t}\bar{\upsilon}}{2(1-\beta_{1,t})(1-\eta)^2(1-\beta_2)}
\\&+\frac{\alpha_t}{\sqrt{\widehat{\upsilon}_{i,t}}} \langle \delta_{i,t},G_{\mathcal{X}}(x_{i,t},\bar{f}_{i,t},\alpha_t) \rangle+ \frac{\bar{\upsilon}\alpha_t}{\underline{\upsilon}^2}\sum_{r=1}^t\beta_1^{t-r}\|\delta_{i,r}\|\| \delta_{i,t}\|,
\end{align}
where the first inequality is due to $ x_{t+1}^*\in \mathcal{X}$ and the optimality of $ x_t^*$ and the second inequality is due to \eqref{35} and the fact that $0\leq\beta_{1,r}<1$. Thus, using \eqref{366} and \eqref{307}, we get
\begin{align*}
  &\Delta_{i,t+1} \leq  \frac{t}{t+1}
  \Big(\Delta_{i,t}-[ \frac{(2-\beta_1)\alpha_t}{2\bar{\upsilon}}-\frac{\rho \alpha_t^2}{2\underline{\upsilon}^2}]\|G_{\mathcal{X}}(x_{i,t},\bar{f}_{i,t},\alpha_t) \|^2
+\frac{\alpha_t\beta_{1,t}\bar{\upsilon}}{2(1-\beta_{1,t})(1-\eta)^2(1-\beta_2)}\\&+\frac{\alpha_t}{\sqrt{\widehat{\upsilon}_{i,t}}} \langle \delta_{i,t},G_{\mathcal{X}}(x_{i,t},\bar{f}_{i,t},\alpha_t) \rangle+ \frac{\bar{\upsilon}\alpha_t}{\underline{\upsilon}^2}\sum_{r=1}^t\beta_1^{t-r}\|\delta_{i,r}\|\| \delta_{i,t}\|\Big)
  +\frac{1}{t+1}(f_{i,t+1}(x_{i,t+1})-f_{i,t+1}(x^*_{t+1})).
  \end{align*}
 By rearranging the above inequality, we have
  \begin{align}\label{090}
   \nonumber\frac{t}{t+1}[ \frac{(2-\beta_1)\alpha_t}{2\bar{\upsilon}}-\frac{\rho \alpha_t^2}{2\underline{\upsilon}^2}]\|G_{\mathcal{X}}(x_{i,t},&\bar{f}_{i,t},\alpha_t)\|^2
   \leq \frac{t}{t+1}\Big(\Delta_{i,t}+\frac{\alpha_t\beta_{1,t}\bar{\upsilon}}{2(1-\beta_{1,t})(1-\eta)^2(1-\beta_2)}\\\nonumber&+\frac{\alpha_t}{\sqrt{\widehat{\upsilon}_{i,t}}} \langle \delta_{i,t},G_{\mathcal{X}}(x_{i,t},\bar{f}_{i,t},\alpha_t) \rangle+ \frac{\bar{\upsilon}\alpha_t}{\underline{\upsilon}^2}\sum_{r=1}^t\beta_1^{t-r}\|\delta_{i,r}\|\| \delta_{i,t}\|\Big)
\\&+\frac{1}{t+1}(f_{i,t+1}(x_{i,t+1})-f_{i,t+1}(x^*_{t+1}))-\Delta_{i,t+1}.
\end{align}
Now, using definition of $\Delta_{i,t+1}$, $\frac{1}{t+1}(f_{i,t+1}(x_{i,t+1})-f_{i,t+1}(x^*_{t+1}))-\Delta_{i,t+1}= -(\frac{t}{t+1})(\bar{f}_{i,t}(x_{i,t+1})-\bar{f}_{i,t}(x^*_{t+1}))$, together with \eqref{090}, we
obtain
\begin{align}\label{388}
  \nonumber  &[ \frac{(2-\beta_1)\alpha_t}{2\bar{\upsilon}}-\frac{\rho \alpha_t^2}{2\underline{\upsilon}^2}] \|G_{\mathcal{X}}(x_{i,t},\bar{f}_{i,t},\alpha_t)\|^2\leq  \Delta_{i,t} -(\bar{f}_{i,t}(x_{i,t+1})-\bar{f}_{i,t}(x^*_{t+1}))\\  &+\frac{\alpha_t\beta_{1,t}\bar{\upsilon}}{2(1-\beta_{1,t})(1-\eta)^2(1-\beta_2)}
  +\frac{\alpha_t}{\sqrt{\widehat{\upsilon}_{i,t}}} \langle \delta_{i,t},G_{\mathcal{X}}(x_{i,t},\bar{f}_{i,t},\alpha_t) \rangle+ \frac{\bar{\upsilon}\alpha_t}{\underline{\upsilon}^2}\sum_{r=1}^t\beta_1^{t-r}\|\delta_{i,r}\|\| \delta_{i,t}\|.
\end{align}

Note that from Assumption~\ref{2030}, we have $\e{\langle \delta_{i,t},G_{\mathcal{X}}(x_{i,t},\bar{f}_{i,t},\alpha_t)\rangle}=0$
and
 $$\e{\|{\boldsymbol\nabla} f_{i,t}(x_{i,t}) -{\nabla} f_{i,t}(x_{i,t}) \|^2}=\e{\|{\boldsymbol\nabla} f_{i,t}(x_{i,t}) \|^2}-\|\e{{\boldsymbol\nabla} f_{i,t}(x_{i,t})} \|^2\leq \e{\| {\boldsymbol\nabla} f_{i,t}(x_{i,t})\|^2}\leq \xi^2,$$
which implies that
\begin{align*}
  \e{\|\delta_{i,t}\|^2} &= \e{\|{\boldsymbol \nabla}\bar{f}_{i,t}(x_{i,t})- \nabla \bar{f}_{i,t}(x_{i,t})\|^2} \\
   &=  \e{\|\frac{1}{t}\sum_{s=1}^t ({\boldsymbol \nabla}f_{i,s}(x_{i,t})-  { \nabla}f_{i,s}(x_{i,t}))\|^2} \\
   &\leq  \frac{1}{t}\sum_{s=1}^t \e{\|{\boldsymbol \nabla}f_{i,s}(x_{i,t})-  { \nabla}f_{i,s}(x_{i,t})\|^2}
   \leq \xi^2.
\end{align*}
The above inequality together with Cauchy-Schwarz for expectations, imply that
\begin{equation}\label{noise:bound}
 \e{\|\delta_{i,t}\|\|\delta_{i,r}\|}\leq \sqrt{\e{\|\delta_{i,r}\|^2}}\sqrt{\e{\|\delta_{i,t}\|^2}}\leq \xi^2.
\end{equation}

Now, using \eqref{noise:bound} and \eqref{808}, summing \eqref{388} over $i\in\mathcal{V}$ and $t\in\{1,\ldots,T\}$, and taking expectation, we obtain
\begin{align}\label{009}
\nonumber &\frac{1}{n}\sum_{i=1}^{n}\min_{t\in\{1,\ldots,T\}} \e{ \|\mathbf{G}_{\mathcal{X}}(x_{i,t},\bar{f}_{i,t},\alpha_t)\|^2}\sum_{t=1}^T[ \frac{(2-\beta_1)\alpha_t}{2\bar{\upsilon}}-\frac{\rho \alpha_t^2}{2\underline{\upsilon}^2}]\\
\nonumber  & \leq \sum_{t=1}^T [ \frac{(2-\beta_1)\alpha_t}{2\bar{\upsilon}}-\frac{\rho \alpha_t^2}{2\underline{\upsilon}^2}]\frac{1}{n}\sum_{i=1}^{n} \e{ \|\mathbf{G}_{\mathcal{X}}(x_{i,t},\bar{f}_{i,t},\alpha_t)\|^2}\\ \nonumber&  \leq
(2+\log T)2L\max_{t\in \{2,\ldots,T\}} \frac{2\sqrt{n}}{(1-\eta)\sqrt{(1-\beta_2)}}\sum_{s=0}^{t-1}\alpha_s\sigma_2^{t-s-1}(W) \\&+\sum_{t=1}^{T}\frac{\alpha_t\beta_{1,t}\bar{\upsilon}}{2(1-\beta_{1,t})(1-\eta)^2(1-\beta_2)}
+\frac{\bar{\upsilon}\xi^2}{\underline{\upsilon}^2}\sum_{t=1}^T \alpha_t\sum_{r=1}^t \beta_1^{t-r}.
\end{align}
Note that $\sum_{t=1}^T[ \frac{(2-3\beta_1)\alpha_t}{2\bar{\upsilon}}-\frac{\rho \alpha_t^2}{2\underline{\upsilon}^2}]>0,$ and
\begin{equation}\label{bound:step}
   \sum_{t=1}^T \alpha_t\sum_{r=1}^t \beta_1^{t-r}=\sum_{t=1}^T \alpha_t\sum_{r=t}^T \beta_1^{r-t}\leq \sum_{t=1}^T  \frac{\alpha_t}{(1-\beta_1)}.
\end{equation}
Therefore, dividing both sides of \eqref{009}
by $\sum_{t=1}^T[ \frac{(2-3\beta_1)\alpha_t}{2\bar{\upsilon}}-\frac{\rho \alpha_t^2}{2\underline{\upsilon}^2}]$, using \eqref{bound:step}, we complete the proof.
\end{proof}

\subsubsection{Proof of Corollary \ref{Corollary10}}\label{C10}

\begin{proof}

Let $\vartheta_t=\sum_{t=1}^T[ \frac{(2-\beta_1)\alpha_t}{2\bar{\upsilon}}-\frac{\rho \alpha_t^2}{2\underline{\upsilon}^2}]$. We claim that if
\begin{equation}\label{bound:tlow1}
T\geq\max\{\frac{4\rho^2\bar{\upsilon}^2}{n\underline{\upsilon}^4(2-\beta_1)^2},\frac{4\bar{\upsilon}^2n}{(2-\beta_1)^2}\},
\end{equation}
then $\vartheta_t \geq 1/2$. One can easily seen that $T$ must satisfy $\frac{(2-\beta_1)\alpha_t}{2\bar{\upsilon}}-\frac{\rho \alpha_t^2}{2\underline{\upsilon}^2}\geq\frac{1}{2T}$ which results in $$-T\rho\alpha_t^2 \bar{\upsilon}+T\underline{\upsilon}^2(2-\beta_1)\alpha_t-\bar{\upsilon}\underline{\upsilon}^2\geq0.$$

By rearranging the above inequality, we obtain
$$1\geq \frac{\rho\alpha_t\bar{\upsilon}}{\underline{\upsilon}^2(2-\beta_1)}+\frac{\bar{\upsilon}}{T(2-\beta_1)\alpha_t}:= A_1+ A_2.$$

Now,  if $A_1, A_2 \leq \frac{1}{2}$, then $\alpha_t\leq \frac{\underline{\upsilon}^2(2-\beta_1)}{2\rho\bar{\upsilon}}$, and  $\alpha_t\geq \frac{2\bar{\upsilon}}{T(2-\beta_1)}$, respectively. These bounds together with our assumption on $\alpha_t$ give \eqref{bound:tlow1}.

Now, let  \eqref{bound:tlow1}  holds.  Using \eqref{388}, and since $\vartheta_t \geq 1/2$, we have
\begin{align}\label{90}
\nonumber &\frac{1}{Tn}\sum_{i=1}^{n}\min_{t\in\{1,\ldots,T\}} \e{ \|\mathbf{G}_{\mathcal{X}}(x_{i,t},\bar{f}_{i,t},\alpha_t)\|^2}\\
\nonumber  & \leq \frac{2}{Tn}\sum_{i=1}^{n}\min_{t\in\{1,\ldots,T\}} \e{ \|\mathbf{G}_{\mathcal{X}}(x_{i,t},\bar{f}_{i,t},\alpha_t)\|^2}\sum_{t=1}^T[ \frac{(2-\beta_1)\alpha_t}{2\bar{\upsilon}}-\frac{\rho \alpha_t^2}{2\underline{\upsilon}^2}]\\
\nonumber  & \leq \frac{2}{T}\sum_{t=1}^T [ \frac{(2-\beta_1)\alpha_t}{2\bar{\upsilon}}-\frac{\rho \alpha_t^2}{2\underline{\upsilon}^2}]\frac{1}{n}\sum_{i=1}^{n} \e{ \|\mathbf{G}_{\mathcal{X}}(x_{i,t},\bar{f}_{i,t},\alpha_t)\|^2}\\ \nonumber&  \leq
\frac{2}{Tn}\sum_{i=1}^{n} \big(f_{i,1}(x_{i,1}) - f_{i,1}(x^*_1) \big)+  L\big(\| x_{i,T+1}-x_{T+1}^*\|+\sum_{t=2}^{T} 2t^{-1}\| x_{i,t}-x^*_{t} \|\big) \\&+\frac{2}{T}\sum_{t=1}^{T}\frac{\alpha_t\beta_{1,t}\bar{\upsilon}}{2(1-\beta_{1,t})(1-\eta)^2(1-\beta_2)}
+\frac{2\bar{\upsilon}\xi^2}{\underline{\upsilon}^2(1-\beta_1)T}\sum_{t=1}^T \alpha_t,
\end{align}
where the last inequality follows from \eqref{asas}, \eqref{807} and  \eqref{noise:bound}.

We proceed to bound RHS of \eqref{90}. By substituting $\alpha_t=\frac{\alpha}{\sqrt{nT}}$ and $\beta_{1,t}=\beta_1\lambda^{t-1}, \lambda\in(0,1)$ into the last two terms of \eqref{90}, we have
\begin{equation}\label{jkj}
  \frac{2\bar{\upsilon}\xi^2}{\underline{\upsilon}^2(1-\beta_1)T}\sum_{t=1}^T \alpha_t =\frac{2\alpha\bar{\upsilon}\xi^2}{\underline{\upsilon}^2(1-\beta_1)\sqrt{nT}},
\end{equation}
 and
\begin{align}\label{sd}
\frac{1}{T} \sum_{t=1}^{T}\frac{\alpha_t\beta_{1,t}\bar{\upsilon}}{(1-\beta_{1,t})(1-\eta)^2(1-\beta_2)} \nonumber &= \frac{\alpha\bar{\upsilon}}{T\sqrt{nT}(1-\eta)^2(1-\beta_2)}
 \sum_{t=1}^{T}\frac{\beta_{1,t}}{(1-\beta_{1,t})} \\
   &= \frac{\alpha\bar{\upsilon}}{T\sqrt{nT}(1-\eta)^2(1-\beta_2)(1-\beta_1)(1-\lambda)}.
\end{align}

Further, using \eqref{90} and Lemma~\ref{32}, we get
 \begin{align}\label{ac}
 \frac{L}{Tn}\sum_{i=1}^{n} \| x_{i,T+1}-x_{T+1}^*\|&\nonumber\leq \frac{2\sqrt{n}\alpha L}{T(1-\eta)\sqrt{(1-\beta_2)(1-\beta_3)}\sqrt{nT}}\sum_{s=0}^{T}\sigma_2^{T-s}(W)\\
 &\leq \frac{2\sqrt{n}\alpha L}{T(1-\eta)\sqrt{(1-\beta_2)(1-\beta_3)}\sqrt{nT}(1-\sigma_2(W))}.
 \end{align}

Similarly, using \eqref{90} and Lemma~\ref{32}, we have
  \begin{align}\label{eq:boundtinv}
  \nonumber
   \frac{L}{Tn}\sum_{i=1}^{n}\sum_{t=2}^{T}t^{-1}\| x_{i,t}-x^*_{t} \| &\leq
   \frac{2\sqrt{n}\alpha L}{\sqrt{nT}(1-\eta)\sqrt{(1-\beta_2)(1-\beta_3)}T}\sum_{t=2}^{T}t^{-1}\sum_{s=0}^{t-1}\sigma_2^{t-s-1}(W)\\
     \nonumber
    &\leq
   \frac{2\sqrt{n}\alpha L}{\sqrt{nT}(1-\eta)\sqrt{(1-\beta_2)(1-\beta_3)}T}\sqrt{\sum_{t=2}^{T}t^{-2}}\sqrt{\sum_{t=2}^{T}(\sum_{s=0}^{t-1}\sigma_2^{t-s-1}(W))^2}\\
   &\leq
   \frac{2\sqrt{n}\alpha L}{\sqrt{nT}(1-\eta)\sqrt{(1-\beta_2)(1-\beta_3)}T (1-\sigma_2(W))},
 \end{align}
where the second inequality is due to Cauchy-Schwarz
inequality and the last inequality follows because $\sum_{t=2}^{T}\frac{1}{t^2} \leq 1$.

Using \eqref{ag}, \eqref{sd}, \eqref{ac} and \eqref{eq:boundtinv} are bounded by \eqref{jkj}. This completes the proof.
\end{proof}

\subsection{Sensitivity of DADAM
to its parameters}\label{sec:sens}

Next, we examine the sensitivity of \textsc{Dadam} on the parameters related to the network connection and update of the moment estimate.

\subsubsection{Choice of the Mixing Matrix}\label{sec:matrices}

In \textsc{Dadam}, the mixing matrices $W$ diffuse information throughout the network. Next, we investigate the sensitivity of \textsc{Dadam} to the parameter of Metropolis constant edge weight matrix $W$ defined as follows~\cite{boyd2004fastest},
      $$
      w_{ij}=\left\{
      \begin{array}{cl}
      \frac{1}{\max\{\degree(i),\degree(j)\}+\iota},&\text{if } (i,j)\in\mathcal{E}, \\
      0,&\text{if } (i,j)\notin\mathcal{E}\text{ and } i\neq j, \\
      1-\sum\limits_{k\in\mathcal{V}} w_{ik},&\text{if } i=j,\\
      \end{array}
      \right.
      $$
      where $\text{deg}(i)$ denote the degree of agent $i$, for some small positive $\iota>0$.

As it is clear from Figure \ref{fig:sensi} the convergence of training accuracy value happens faster for sparser networks (higher $\sigma_2(W)$). This is similar to the trend observed for FedAvg algorithm while reducing parameter $C$ which makes the agent interaction matrix sparser. This is also expected as discussed in Theorems~\ref{regret} and \ref{theorem2}. Note that with the availability of a central parameter server (as in FedAvg algorithm), sparser topology may be useful for a faster convergence, however, topology density of graph is important for a distributed learning scheme with decentralized computation on a network.
\begin{figure}[h]
	\centering \makebox[0in]{
\begin{tabular}{c}
\includegraphics[scale=0.43]{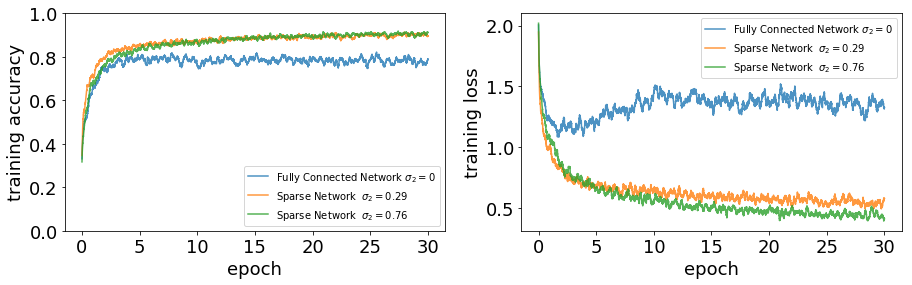}
\end{tabular}}
\caption{Performance of the \textsc{Dadam} algorithm with varying network topology: training loss and accuracy over 30 epochs based on the \texttt{MNIST} digit recognition library.}\label{fig:sensi}
\end{figure}

\subsubsection{Choice of the Exponential Decay Rates}\label{sec:decay}

We also empirically evaluate the effect of the $\beta_3$ in Algorithm~\ref{alg}. We consider a range of hyper-parameter choices, i.e. $\beta_3 \in \{0, 0.9 ,0.99\}$.
\begin{figure}[h]
	\centering \makebox[0in]{
\begin{tabular}{c}
\includegraphics[scale=0.43]{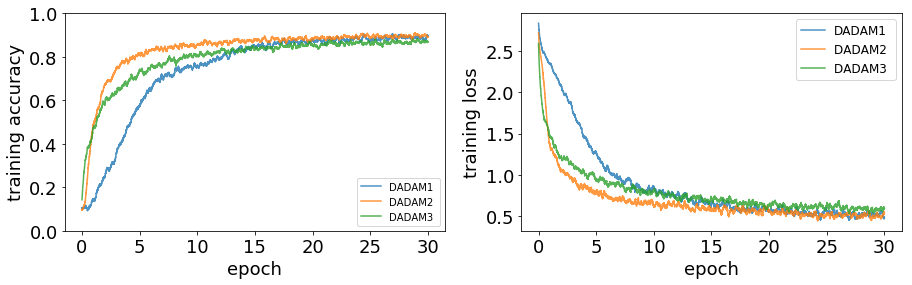}
\end{tabular}}
\caption{Performance of the \textsc{Dadam} algorithm with varying decay rate $\beta_3$. \textsc{Dadam}1~($\beta_3=0$), \textsc{Dadam}2~($\beta_3=0.9$), and \textsc{Dadam}3~($\beta_3=0.99$) for training loss and accuracy over 30 epochs based on the \texttt{MNIST} digit recognition library.}\label{fig:dadambeta}
\end{figure}
From Figure~\ref{fig:dadambeta} it can be easily seen that \textsc{Dadam} performs equal or better than \textsc{Amsgrad} ($\beta_3=0$), regardless of the hyper-parameter setting for $\beta_1$ and $\beta_2$.

%

\end{document}